%% file: main.tex
\documentclass{article}

\usepackage{microtype}
\usepackage{graphicx}
\usepackage{booktabs} 

\usepackage{hyperref}

\usepackage[preprint]{icml2026}

\usepackage{amsmath}
\usepackage{amssymb}
\usepackage{mathtools}
\usepackage{amsthm}
\usepackage{subfig}
\usepackage{verbatim,float,url,enumerate}
\usepackage{psfrag}
\usepackage{xcolor,dsfont}
\usepackage{cancel}
\usepackage[abbreviations]{foreign}
\usepackage{subcaption}
\usepackage{array}
\usepackage{multicol,multirow}
\usepackage{tikz}
\usepackage{pgfplots}
\usepackage{soul}
\usepackage{longtable}
\usepackage{ltablex}
\usepackage{listings}
\usepackage{dsfont}
\usepackage{color}
\input{math_commands.tex}
\usepackage{makecell}
\usepackage{pifont}
\usepackage{enumitem}

\usepackage[capitalize,noabbrev]{cleveref}

\usepackage[textsize=tiny]{todonotes}

\icmltitlerunning{Matching $f$-divergences for Fairness in Generative Models}

\begin{document}

\twocolumn[
  \icmltitle{Equalized Generative Treatment: Matching $f$-divergences \\ for Fairness in Generative Models}



  \icmlsetsymbol{equal}{*}

  \begin{icmlauthorlist}
    \icmlauthor{Alexandre V\'erine}{yyy}
    \icmlauthor{Rafael Pinot}{comp}
    \icmlauthor{Florian Le Bronnec}{xxx}
  \end{icmlauthorlist}

  \icmlaffiliation{yyy}{DI ENS, École normale supérieure, Université PSL, CNRS, Paris, France}
  \icmlaffiliation{comp}{Sorbonne Universit\'e, Universit\'e Paris Cit\'e, CNRS, Laboratoire de Probabilit\'es, Statistique et Mod\'elisation, LPSM, Paris, France}
  \icmlaffiliation{xxx}{Miles, LAMSADE, Universit\'e  Paris-Dauphine-PSL, CNRS, Paris, France}

  \icmlcorrespondingauthor{Alexandre V\'erine}{alexandre.verine@ens.fr}

  \icmlkeywords{Machine Learning, ICML}

  \vskip 0.3in
]



\printAffiliationsAndNotice{}  
\begin{abstract}
  Fairness is a crucial concern for generative models, which not only reflect but can also amplify societal and cultural biases. Existing fairness notions for generative models are largely adapted from classification and focus on balancing the probability of generating samples from each sensitive group. We show that such criteria are brittle, as they can be met even when different sensitive groups are modeled with widely varying quality. To address this limitation, we introduce a new fairness definition for generative models, termed as \emph{equalized generative treatment} (EGT), which requires comparable generation quality across all sensitive groups, with quality measured via a reference $f$-divergence. We further analyze the trade-offs induced by EGT, demonstrating that enforcing fairness constraints necessarily couples the overall model quality to that of the most challenging group to approximate. This indicates that a simple yet efficient min–max fine-tuning method should be able to balance $f$-divergences across sensitive groups to satisfy EGT. We validate this theoretical insight through a set of experiments on both image and text generation tasks. We demonstrate that min-max methods consistently achieve fairer outcomes compared to other approaches from the literature, while maintaining competitive overall performance for both tasks.
\end{abstract}

\input{content/intro}

\input{content/fdivergences}
\input{content/brittleness}
\input{content/equalized}
\input{content/experiments}

\input{content/conclusion}

\section*{Acknowledgments}
 This work was granted access to the HPC resources of IDRIS under the allocations 2025-AD011016159R1 and 2025-A0181016159  made by GENCI.

\input{content/appendix/statement}

\bibliography{icml2026_conference,references}
\bibliographystyle{icml2026}

\newpage
\appendix
\onecolumn

\input{content/appendix/mathsupp}

\input{content/appendix/experiment}

\end{document}

%% file: math_commands.tex
\usepackage{amsmath,amsfonts,bm}
\usepackage{amsthm}

\theoremstyle{plain}
\newtheorem{theorem}{Theorem}[section]
\newtheorem{lemma}[theorem]{Lemma}

\theoremstyle{definition}
\newtheorem{definition}[theorem]{Definition}

\theoremstyle{remark}

\newcommand{\Supp}{\textsc{Supp}}

\def\eqref#1{equation~\ref{#1}}

\def\1{\bm{1}}

\DeclareMathAlphabet{\mathsfit}{\encodingdefault}{\sfdefault}{m}{sl}
\SetMathAlphabet{\mathsfit}{bold}{\encodingdefault}{\sfdefault}{bx}{n}

\newcommand{\E}{\mathbb{E}}

\newcommand{\R}{\mathbb{R}}

\newcommand{\PX}{\mathcal{P}_\lambda \left( \mathcal{X} \right)}

\newcommand{\Df}{\mathcal{D}_f}
\newcommand{\Q}{\mathcal{Q}}

\DeclareMathOperator*{\argmax}{arg\,max}
\DeclareMathOperator*{\argmin}{arg\,min}

\newcommand{\cmark}{\ding{51}}%
\newcommand{\xmark}{\ding{55}}%

%% file: content/intro.tex
\section{Introduction}
\label{sec:intro}

The rise of generative models has revolutionized a wide range of domains, including natural language processing, computer vision, and scientific discovery~\cite{hu_simulating_2025}. 
As these systems become ever more ubiquitous, ensuring their \emph{trustworthiness} has emerged as a critical challenge~\citep{Kucharavy2024}. Trustworthiness is inherently multidimensional, encompassing many challenges such as robustness to perturbations~\citep{carlini2024aligned}, protection of user privacy~\citep{carlini_extracting_nodate,carlini_quantifying_2023}, and fairness \citep{choi_fair_2020,zameshina_fairness_2022}. Among them, concerns about fairness are particularly pressing, as generative models not only reflect but also shape the ways information and cultural artifacts are created, distributed, and consumed. Left unchecked, these systems risk amplifying and existing societal biases and stereotypes. Despite growing interest, most approaches to fairness in generative models fail to provide definitions that reflect the unique characteristics of generative modeling. Existing metrics are often adapted directly from the literature on fair classification, focusing on recalibrating generative odds (i.e., the probability of being generated) across sensitive groups~\citep{choi_fair_2020,zameshina_fairness_2022,teo_fair_2022,teo_measuring_2023}. However, such criteria fall short in capturing nuanced disparities on how different subpopulations are modeled.

This apparent lack of formalization in addressing fairness issues may stem from the broader challenge of assessing the quality of generative models, which remains an open problem. At its core, generative modeling seeks to approximate a target data distribution using a (parametrized) model distribution. A common strategy involves minimizing a statistical $f$-divergence, where the choice of $f$ defines the objective and, consequently, the properties of the trained model \cite{goodfellow_generative_2014,nowozin_f-gan_2016,grover_flow-gan_2018}.
At evaluation, however, the process becomes more complex due to the coexistence of a multitude of competing metrics. While $f$-divergences can off course  be employed for evaluation (e.g., for computing the precision and recall of the model \cite{verine_precision-recall_2023}) other metrics such as the FID \cite{heusel_gans_2017}, MAUVE \cite{pillutla_mauve_2023}, or self-BLEU \cite{zhu_texygen_2018}), are also frequently used in practice. These alternatives, however, often represent ad hoc choices with limited or no theoretical grounding, offering little insight to analysts.
To establish a clear and general framework for analyzing fairness in generative models, we adopt $f$-divergences. This choice ensures consistency with existing training objectives while providing a principled foundation for our fairness analysis, as detailed further in Section~\ref{def:fdiv}. In this context, our contributions are as follows.

\textbf{Brittleness of existing fairness criterion.} We first show in Theorem~\ref{th:failure} that existing notions of fairness can yield models that, while well-calibrated in terms of generative odds, can be arbitrarily unbalanced in terms of $f$-divergence between the target and trained conditional distribution for each sensitive group. This means that even when satisfying these definitions, models can generate some groups with much higher quality and diversity than some others. This outcome may not seem surprising, as existing methods were not designed to address this issue directly. Nevertheless, it underscores how the current lack of formalization can result in approaches that provide a false sense of fairness and provide a formal proof of the brittleness of existing definitions. We validate this theoretical result with numerical experiments for image generation on the FFHQ dataset~\citep{karras_analyzing_2020} or for text on the Wikipedia Biographies dataset~\citep{bronnec_exploring_2024}. In each case, satisfying existing definitions do not prevent the model from treating each sensitive group differently in terms of generation quality.

\textbf{A new definition matching $f$-divergences.} Based on this observation, we introduce a new definition of fairness called \emph{equalized generative treatment} (EGT), which relies on simultaneously controlling the $f$-divergence between the target and trained distributions for each sensitive group. In Theorem~\ref{th:lowerbound}, we demonstrate that applying this definition promotes the minimization of the highest $f$-divergence among sensitive groups. This naturally indicates that min-max fine-tuning methods are good candidates to target our new fairness criterion. We compare these methods to existing fairness solutions in the literature, evaluating performance on image generation using the FFHQ dataset \cite{karras_analyzing_2020} and on text generation using the Wikipedia Biographies dataset \cite{bronnec_exploring_2024}. While some methods may coincidentally improve EGT for specific model/task combinations, min-max schemes consistently outperform or matches prior work on this criterion across all settings we consider. In short, we show that min-max fine-tuning methods provide a stable and theoretically grounded solution to the fairness problem in generative models.

%% file: content/fdivergences.tex
\section{Background \& Related Works}

Let $\mathcal{X} \subset \R^d$, endowed with the euclidean norm $\Vert \cdot \Vert$. We denote $\PX$ the set of probability distribution on $\mathcal{X}$ that are absolutely continuous w.r.t. the Lebesgue measure $\lambda$. For any $P \in \PX$, its probability density function (pdf) is denoted by $p = \tfrac{dP}{d\lambda}$. Finally, we denote $\Delta(\mathcal{A})$ the probability simplex over a finite set $\mathcal{A}$, and $\mathds{1}_E(x)$ the indicator function of the event $x \in E$, for any $E \subset \mathcal{X}$.

\subsection{$f$-divergences in generative modeling}

Learning a generative model can be formalized as approximating a target distribution $P \in \PX$ with a model distribution $Q$, where $Q$ belongs to an admissible family
$\Q \subset \PX$. 
To measure the discrepancy between $P$ and $Q$, we use the
general class of $f$-divergences, defined below. \vspace{2pt}

\begin{definition}
    \label{def:fdiv}
    Let $f: (0, +\infty) \to ( - \infty, + \infty ]$ be a convex, lower semi-continuous
    function with $f(1) = 0$. For any $P, Q \in \PX$, the $f$-divergence between the distributions $P$
    and $Q$ is defined as
    $$
        \footnotesize
        \mathcal{D}_f(P\|Q) =  \int_{\Supp(Q)} f\!\left(\tfrac{p(x)}{q(x)}\right) q(x)\, d\lambda(x)
        + \zeta(\Supp(Q)),
    $$
    where $p$ and $q$ are the pdfs of $P$ and $Q$,
    $\Supp(Q) = \{ x \in \mathcal{X} \mid q(x) > 0\}$ denotes the support of $Q$ , $\zeta(\Supp(Q)) = \bar{f}(\infty)\, P \left( \mathcal{X} \backslash \Supp(Q) \right)$, and
    $\bar{f}(\infty) = \lim_{t \to +\infty} f(t)/t$. Furthermore, this definition adopts the conventions  $f(0) = \lim_{t \to 0^+} f(t)$ and $0 \times \infty = 0$ to avoid the above being ill-defined when $\bar{f}(\infty) = \infty$  and $P \ \left(\mathcal{X} \setminus \Supp(Q) \right) = 0$.
\end{definition}

The choice of $f$ in Definition~\ref{def:fdiv} specifies the divergence we aim to compute. For instance, the Kullback-Leibler divergence can be obtained by setting $f(t) = t \log t$. This divergence is usually minimized by likelihood-based methods used in LLMs \cite{grattafiori_llama_2024}, in Normalizing Flows \citep{rezende_variational_2016} or in some score-based diffusion models \citep{song_maximum_2021}. On the other hand, Generative Adversarial Networks \citep{goodfellow_generative_2014} typically optimize objective functions related to the Jensen-Shannon divergence, defined for $f(t) = t \log t - (t+1) \log (t+1)/2$. Additionally, several methods have also been proposed to minimize the Total Variation in GANs \citep{um_fair_2021} or LLMs \citep{ji_tailoring_2023}. More generally, recent work has proposed modular frameworks that allow targeting a variety of $f$-divergences, enabling practitioners to tailor the objective to the specific context and application image modeling~\citep{nowozin_f-gan_2016,grover_flow-gan_2018,cai_utilizing_2020,verine_precision-recall_2023,xu_one-step_2025} and more recently for LLMs \citep{wang_beyond_2023,sun_inverse-rlignment_2024,go_aligning_2023,verine_improving_2025}.



%




\medskip
\textbf{Evaluation.} In practice, there exists many metrics to evaluate the performance of a generative model, each with their own strengths and weaknesses~\citep{borji_pros_2022}. 
Among these metrics, a prominent example is the family of metrics that separately measure quality and diversity, most notably \emph{precision} and \emph{recall} for generative models \cite{kynkaanniemi_improved_2019} refined by \citet{kim_toppr_2023}. Given two distributions $P$ and $Q$, these metrics are defined as
\begin{itemize}
    \item $\mathrm{Precision}(Q \| P) = Q \ \!\bigl( \Supp(P) \bigr),$
    \item $\mathrm{Recall}(Q \| P) = P \ \!\bigl( \Supp(Q) \bigr).$
\end{itemize}

Precision and recall can be interpreted as a particular instance of $f$-divergences as demonstrated by~\cite{simon_revisiting_2019,verine_quality_2024}. Furthermore, extensions and generalizations of these metrics building on $f$-divergences have been investigated in several recent works~\citep{sajjadi_assessing_2018,djolonga_precision-recall_2020,pillutla_mauve_2023}. Furthermore, these two metrics are often combined within a $\mathrm{Precision/Recall}$ divergence which is simply defined as the sum of the $\mathrm{Precision}(Q \| P) + \mathrm{Recall}(Q \| P)$. Note that, by additivity of $f$-divergences, $\mathrm{Precision/Recall}$ is also a $f$-divergence (as a sum of $f$-divergences).


\subsection{Fairness in Generative Modeling}
\label{sec:fairness}

In the context of fairness, we extend this problem formulation by introducing a set of sensitive attributes $\mathcal{A}$ together with an oracle function $\psi : \mathcal{X} \to \mathcal{A}$ that maps each instance $x \in \mathcal{X}$ to its corresponding sensitive attribute. We assume $\psi$ is defined over the entire space, i.e., $\text{dom}(\psi) = \mathcal{X}$. This induces a partition of $\mathcal{X}$ into disjoint subsets $\mathcal{X}_a = \{ x \in \mathcal{X} \mid \psi(x) = a \}$ for each $a \in \mathcal{A}$. For any distribution $P \in \PX$ with pdf $p$, we can then express $P$ as the mixture
\[
    P = \sum_{a \in \mathcal{A}} \pi_a^P P_a \quad \mbox{with} \quad p_a(x) = \frac{p(x)}{\pi_a^P}\,\mathds{1}_{\mathcal{X}_a}(x) \quad \forall x \in \mathcal{X},
\]
where $\pi_a^P = P(\mathcal{X}_a)$ denotes the proportion of the population associated with attribute $a$, and $P_a$ is the conditional distribution restricted to $\mathcal{X}_a$. Thus, the vector $(\pi_a^P)_{a\in\mathcal{A}}$ lies in the simplex $\Delta(\mathcal{A})$, and $P$ can be interpreted as a mixture of attribute-conditional distributions. We call a distribution $P \in \PX$ non-trivial if $\pi_a^P > 0$ for all $a \in \mathcal{A}$. In the following, we will always assume the target distribution we consider is non-trivial, since fairness considerations would otherwise be meaningless.

\textbf{Existing fairness criteria.}
Most existing approaches to fairness in generative modeling focus on the proportions of sensitive attributes in the generated distribution. This perspective was popularized by \citet{hutchinson_50_2019}, who characterized fairness as the equal representation of a chosen sensitive attribute among generated samples. To the best of our knowledge, no standard terminology has been established in the literature for such proportion-based criteria. We therefore introduce the following nomenclature: \emph{equalized generative odds} (EGO), and \emph{matching generative odds} (MGO). The distinction between these two notions reflects different goals: EGO enforces uniform representation across groups, whereas MGO requires the generative model $Q$ to reproduce the proportions of the target distribution $P$. \medskip



\begin{definition}
    \label{def:GO}
    Let $P, Q \in \PX$ and $\delta > 0$. $Q$ is said to satisfy $\delta$-\emph{equalized generative} odds if
    \[
        \bigl| \pi_a^Q - \pi_{a'}^Q \bigr| \leq \delta,
        \quad \text{for all } a,a' \in \mathcal{A}.
    \]
    When $\delta = 0$, we say that $Q$ satisfies \emph{equalized generative odds}. Furthermore, $P$ and $Q$ are said to satisfy $\delta$-\emph{matching generative odds} if
    \[
        \bigl| \pi_a^Q - \pi_a^P \bigr| \leq \delta,
        \quad \text{for all } a \in \mathcal{A}.
    \]
    When $\delta = 0$, we say that $P$ and $Q$ satisfy \emph{matching generative odds}, meaning that the group proportions under $Q$ exactly match those of $P$.
\end{definition}


\textbf{Existing methods.}
Most approaches to fairness in generative modeling have focused on enforcing group proportions, either through EGO or MGO. A variety of methods have been developed with this goal in mind. For instance, \citet{choi_fair_2020} reweight the training distribution, enforcing EGO when the data is unbiased and approximating MGO otherwise. Since method has since been widely adopted and extended in the literature \cite{yazdani-jahromi_fair_2024,yan_forml_2022,kim_pfguard_2025} and adapted to diffusion models by \citet{kim_training_2024}. Other works modify the generation process itself, thereby directly controlling the proportion. For instance, \citet{frankel_fair_2018,humayun_magnet_2022,tan_improving_2021} regularize the latent space to reduce deviations from MGO, while \citet{zameshina_fairness_2022} apply rejection sampling based on predicted sensitive attributes, which can be tuned to enforce either EGO or MGO. Specifically for diffusion, \citet{parihar_balancing_2025} focus on using guidance to ensure EGO during training.

On the evaluation side, \citet{teo_measuring_2023} proposed classifier-based metrics designed to assess how well generative models satisfy EGO or MGO, but again the focus remained solely on proportions. Only very recently has work begun to go beyond this perspective. Other work such as \citet{mayer_improving_2024} and \citet{um_fair_2021} introduced an evaluation based on the FID between majority and minority groups, although this was limited to the case of synthetic data generation or GANS specific approach. In summary, while most prior work has focused on improving or evaluating fairness through proportions (EGO or MGO), little attention has been given to local distributional discrepancies, which can be more effectively captured by $f$-divergences for both evaluation and training.

%% file: content/brittleness.tex
\section{On the brittleness of existing definitions}
\label{sec:brittleness}
\begin{figure*}[!ht]
    \newcommand{\figheight}{92pt}
    \subfloat[Models with identical $\mathcal{D}_{\mathrm{JS}}(P\Vert Q)$]{\includegraphics[height=\figheight]{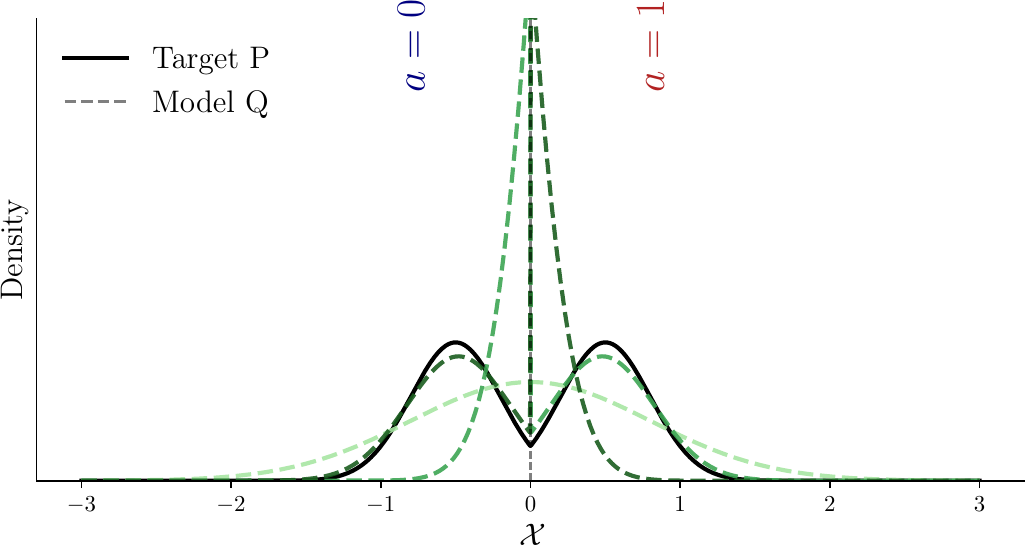}\label{fig:brittle1}}
    \hfill
    \subfloat[$\mathcal{D}_{\mathrm{JS}}(P\Vert Q)$]{\includegraphics[height=\figheight]{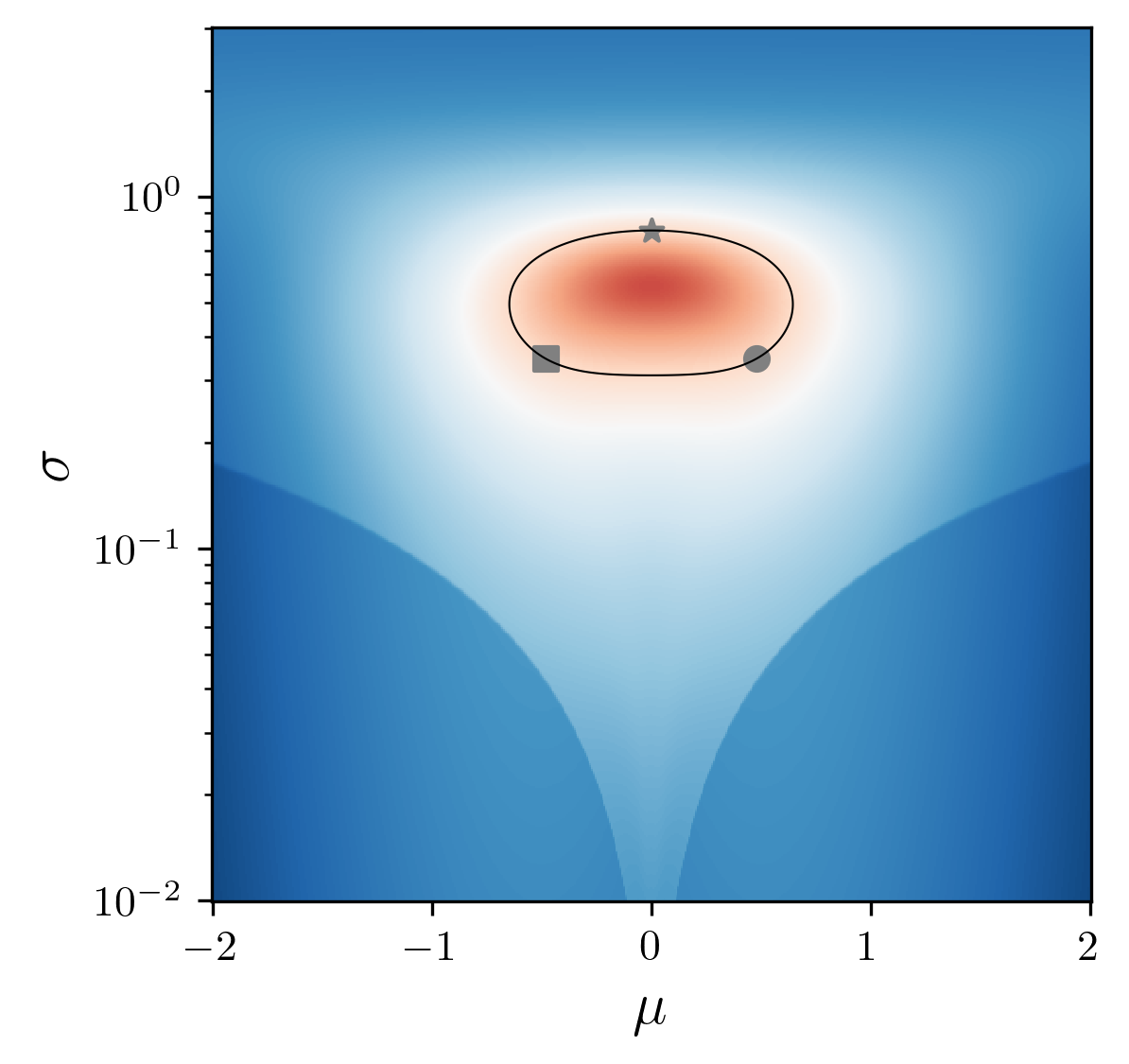}\label{fig:brittle2}}
    \hspace{.5em}
    \subfloat[$\mathcal{D}_{\mathrm{JS}}(P_0\Vert Q_0)$]{\includegraphics[height=\figheight]{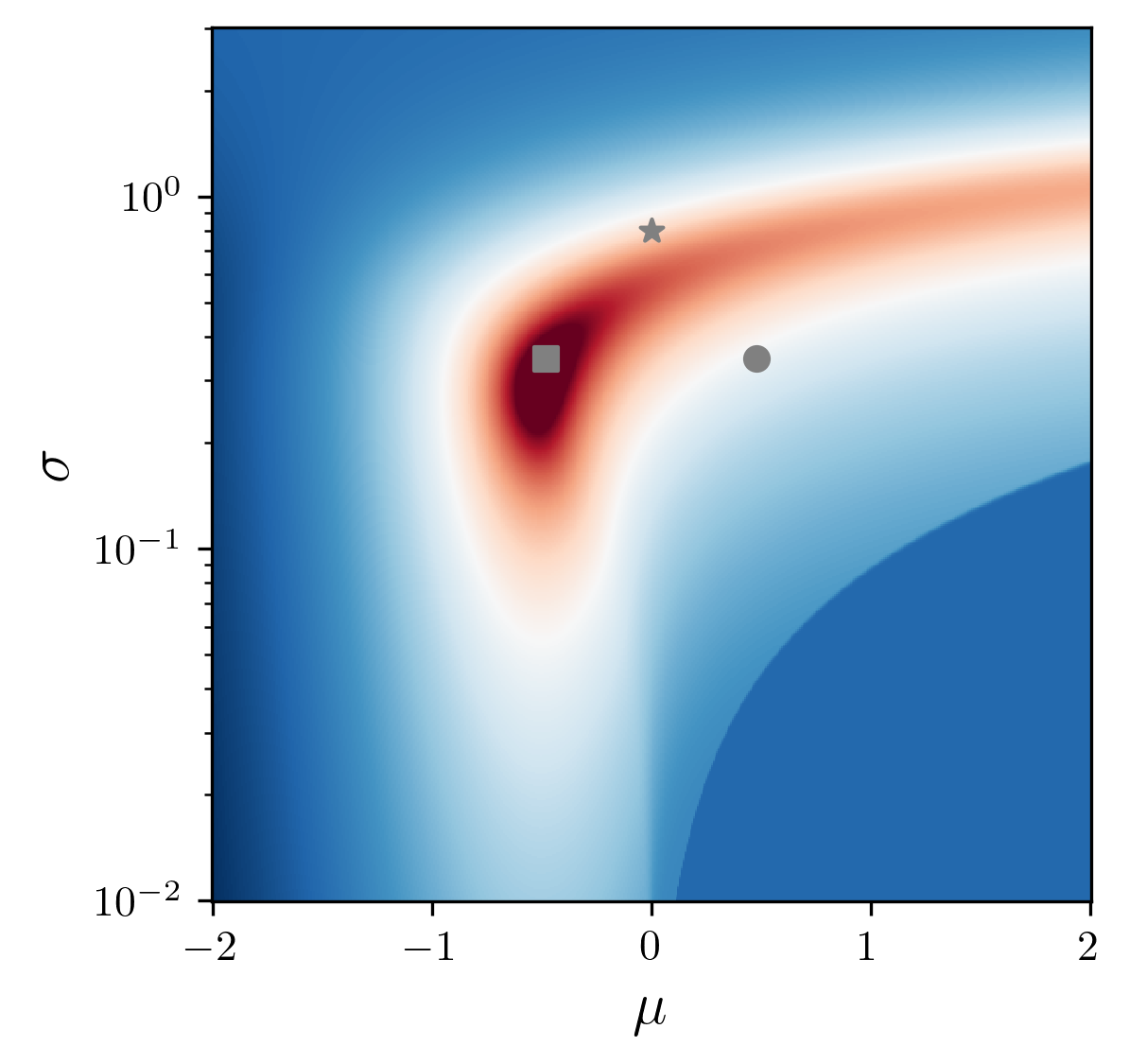}\label{fig:brittle3}}
    \hspace{.5em}
    \subfloat[$\mathcal{D}_{\mathrm{JS}}(P_1\Vert Q_1)$]{\includegraphics[height=\figheight]{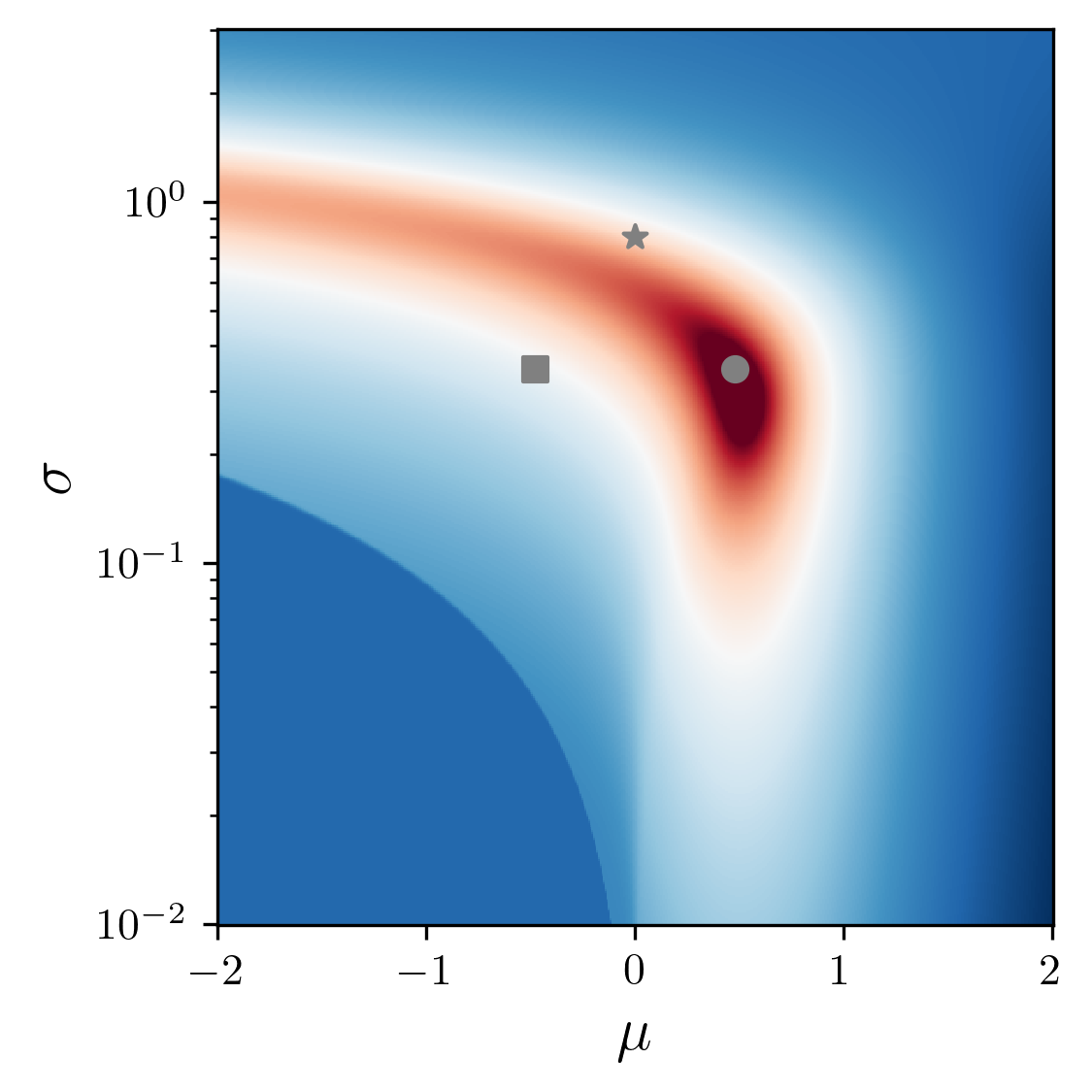}\label{fig:brittle4}}
    \caption{Jensen–Shannon divergence between a target distribution $P$ and rescaled Gaussian models $Q=\mathcal{N}(\mu,\sigma^2)$.
    (\ref{fig:brittle1}) Models with the same global divergence $\mathcal{D}_{\mathrm{JS}}(P\Vert Q)$ can still differ greatly.
    (\ref{fig:brittle2}) Level set for $\mathcal{D}_{\mathrm{JS}}=1$, with selected models marked by stars.
    (\ref{fig:brittle3})--(\ref{fig:brittle4}) Conditional  divergences for the two groups, models on the same level set may yield highly unbalanced conditional divergences.}
    \label{fig:brittle}
\end{figure*}
In Section~\ref{sec:fairness}, we discussed how most existing approaches to fairness in generative modeling are framed in terms of \emph{matching generative odds} (MGO) or \emph{equalized generative odds} (EGO). These notions capture fairness only at the level of group proportions, without accounting for the local behaviors (namely, how accurately the conditionals $(P_a)_{a\in\mathcal{A}}$ are reproduced). In this section, we argue that such a focus on proportions is inherently fragile. In particular, we show that even when MGO and EGO are perfectly satisfied, substantial discrepancies can (and does) remain in how different sensitive groups are modeled. We first illustrate this limitation through a simple example, then formalize it in a general theoretical result, and finally validate it empirically on state-of-the-art models.

\subsection{Theoretical Brittleness of Matching and Equalized Generative Odds}

As a warm-up, consider the setting illustrated in Figure~\ref{fig:brittle}, where $\mathcal{X} = \mathbb{R}$, $\mathcal{A} = \{0,1\}$, and the oracle function is $\psi = \mathds{1}_{\mathbb{R}_{+}}$. The target distribution is defined as $P = \tfrac{1}{2} P_0 + \tfrac{1}{2} P_1$, where $P_0$ and $P_1$ are truncated Gaussian distributions with respective means $\pm 0.5$ and standard deviation $0.3$. As model family $\mathcal{Q} $, we consider the set of univariate Gaussian distributions $\{ \mathcal{N}(\mu,\sigma^2) \mid (\mu,\sigma) \in \mathbb{R} \times \mathbb{R}_{+} \}$ that we rescale on each side of the origin to enforce MGO and EGO. To examine the potential disparities arising under these criteria, we focus on the Jensen Shannon divergence $\mathcal{D}_{\mathrm{JS}}$. For this divergence, we study the level set $\{ Q \in \mathcal{Q} \mid \mathcal{D}_{\mathrm{JS}}(P\|Q) = \epsilon \}$, for a fixed level $\epsilon = 1$.


Figure~\ref{fig:brittle2} shows that this level set allows for several pairs of admissible parameters $(\mu,\sigma) \in \mathbb{R} \times \mathbb{R}_{+} $, with three representative solutions highlighted by points marked with a star. On the one hand, for the upper star in Figure~\ref{fig:brittle2}, the conditional divergences $\mathcal{D}_{\mathrm{JS}}(P_0\|Q_0)$ and $\mathcal{D}_{\mathrm{JS}}(P_1\|Q_1)$ are both small and almost identical. On the other hand, at the left and right stars in Figure~\ref{fig:brittle2}, a clear imbalance appears. Indeed, in both cases, one of the conditional divergences is as low as $0.02$, while the other reaches $1.98$. This illustrative example shows that even when MGO and EGO are perfectly satisfied, the model can still exhibit arbitrarily poor fidelity for one of the sensitive groups.

\textbf{More general result.} The brittleness observed above is not a mere artifact of our toy example, but rather a general phenomenon. Even when $\mathcal{Q}$ is allowed to range over all distributions in $\PX$ that satisfy EGO and MGO with $P$, and even when the global $f$-divergence $\Df(P\|Q)$ is constrained to be arbitrarily small (but non-zero), it remains possible to design situations in which one group incurs an arbitrarily larger conditional divergence than the others. To make this precise, by analogy with the warm-up above, we introduce the \emph{$f$-divergence level set} around $P$ as
\[
    \mathcal{S}_{\Df}(P,\epsilon)
    \;=\;\bigl\{\, Q \in \PX \;\big|\; \Df(P\|Q) = \epsilon \,\bigr\}.
\]
By working within $\PX$, we impose no restriction on model expressivity, thereby going beyond the specificity of our toy example. In this setting, the phenomenon of imbalanced conditional divergences across attributes can be formalized as follows.

\begin{theorem}
    Let $P \in \PX$ be a non-trivial target distribution satisfying EGO, and let $f$ be a continuous function such that $\Df$ defines an $f$-divergence. For any $\epsilon \in (0, f(0) + \bar{f}(+\infty))$ and any $\gamma \in (0,\epsilon)$, there exists $Q^{\gamma} \in \mathcal{S}_{\Df}(P,\epsilon)$ such that $Q^{\gamma}$ satisfies MGO with $P$, but for which there exists $\bar{a} \in \mathcal{A}$ with
    \[
        \Df(P_{\bar{a}} \| Q^{\gamma}_{\bar{a}}) \;\geq\; \Df(P_{a} \| Q^{\gamma}_{a}) + \gamma,
        \qquad \forall a \in \mathcal{A}\setminus\{\bar{a}\}.
    \]
\end{theorem}
In other words, even when both EGO and MGO are met, it is always possible to build a model performing arbitrarily worse on one group than on the others. This establishes that global fairness criteria based solely on proportions provide no guarantee of conditional quality for the groups.

\begin{figure*}[t!]
    \newcommand{\figheightdiff}{110pt}
    \centering
    \subfloat[EDM under MGO or EGO constraints.\label{fig:ffhq_mgo_vp}]{
        \includegraphics[height=\figheightdiff]{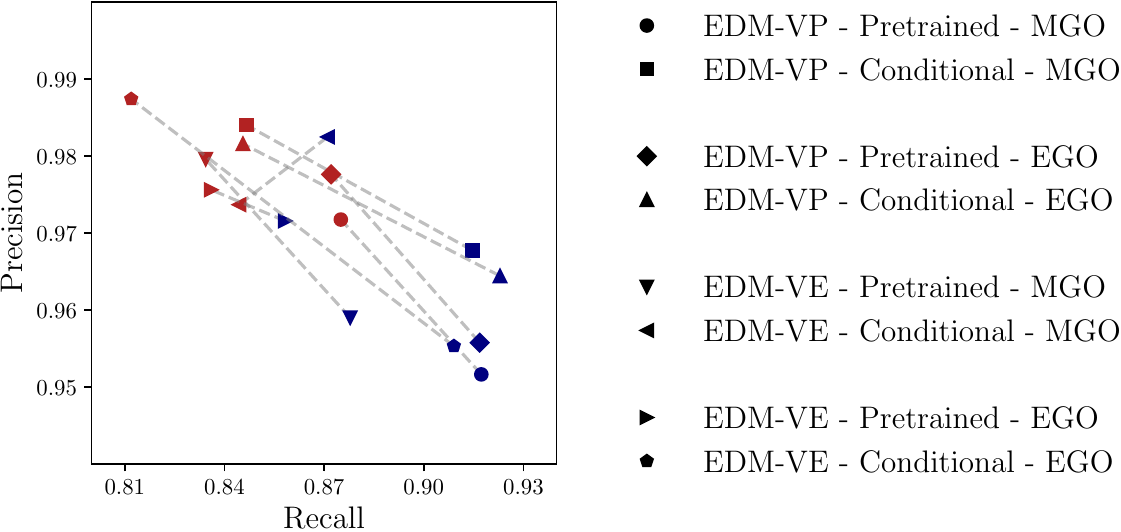}}\hfill
    \subfloat[LLaMA-3.2-Chat \& Gemma-3 under MGO constraints.\label{fig:llama_chat}]{
        \includegraphics[height=\figheightdiff]{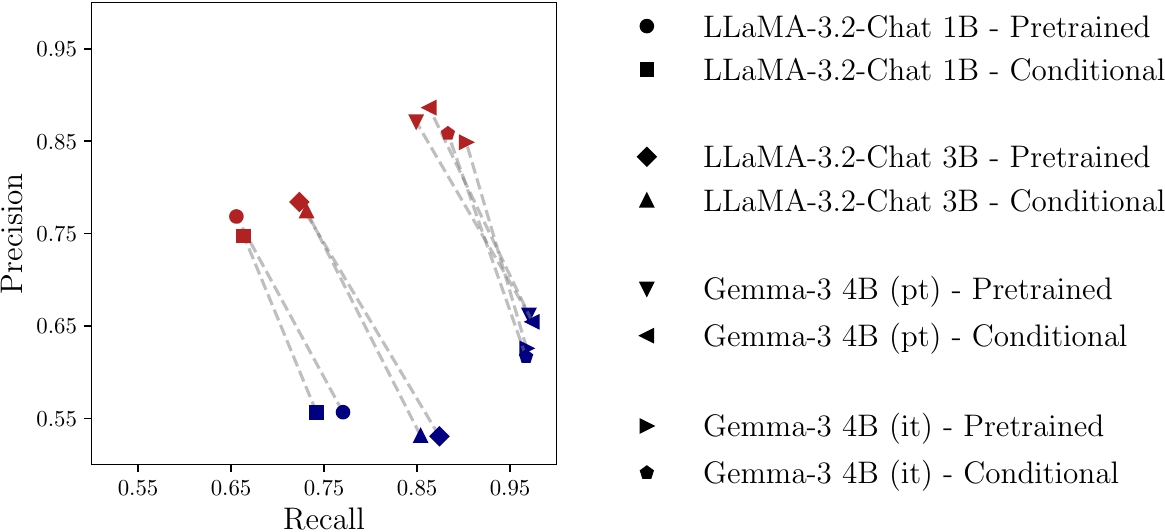}}
    \caption{Precision and recall for EDM (VE and VP) on FFHQ (\ref{fig:ffhq_mgo_vp}) and for LLaMA-3.2-Chat (1B and 3B) and Gemma-3 4B (pretrained and instructed-tuned) on the Wikipedia Biographies (\ref{fig:llama_chat}) under two settings: pretrained and conditional. At the sampling stage, rejection sampling is used to enforce either MGO or EGO. Each sensitive-group is color coded (red or blue) and the points corresponding to the subgroups for a given model are connected to each other by a dashed line. We observe significant discrepancies in precision and recall persist across groups, demonstrating the brittleness of proportion-based definitions.}
    \label{fig:ffhq}
\end{figure*}
\subsection{Observing the brittleness in practice}

The brittleness highlighted above is not only theoretical: it also manifests in practice with off-the-shelf generative models. To demonstrate this, we run experiments on competitive diffusion models (images) and large language models (text). Experimental details are provided below, with additional information in Appendix~\ref{app:sec:experiments}.

\textbf{Datasets and sensitive attributes.} Two modalities are considered: images and text. For image generation, we consider FFHQ \citep{karras_style-based_2019} and focus on gender as the sensitive attribute (44\% male, 56\% female). For text generation, experiments rely on a dataset derived from Wikipedia Biographies \citep{bronnec_exploring_2024}, again using gender as the sensitive attribute (approximately 75\% male, and 25\% female respectively).

\textbf{Methods.} An unconditional baseline (pretrained by the original authors) is compared to a conditional baseline that explicitly conditions on the sensitive attribute, enabling direct control of group proportions. For a fair comparison, rejection sampling is applied to the unconditional model so that generated samples match the desired sensitive-group proportions \cite{zameshina_fairness_2022}. MGO or EGO is enforced for the image task, while only MGO is enforced for the text task due to computational constraints.

\textbf{Models.} Image generation relies on Elucidated Diffusion Models (EDM) \citep{karras_elucidating_2022} with both variance-preserving (VP) and variance-exploding (VE) parameterizations. The unconditional baseline uses the authors' pretrained checkpoints, while the conditional model is obtained by fine-tuning on FFHQ using class-conditional embeddings. Text generation uses LLaMA-3.2-Chat \citep{grattafiori_llama_2024} (1B and 3B) and Gemma-3 4B \citep{gemma_team_gemma_2025} in both pretrained (pt) and instruction-tuned (it) variants. Conditional generations are produced by prompting the base models with the sensitive attribute.

\textbf{Evaluation metrics and oracle functions.} Group-wise performance is assessed via precision and recall per sensitive group using Topological Precision and Recall (TopP\&R) \citep{kim_refining_2023}. This procedure requires an embedding model and an oracle function $\psi$ defining the sensitive groups. For FFHQ, the oracle is a fine-tuned DINOv2 ViT \cite{oquab_dinov2_2024} gender classifier, and its embeddings are used for TopP\&R. For Wikipedia Biographies, a keyword-based oracle is enough to perfectly classify the attributes, and we used Qwen3-4B text-embeddings \citep{zhang_qwen3_2025}.

\textbf{Observed disparities.} Figure~\ref{fig:ffhq} reports precision and recall by sensitive group. Despite enforcing MGO (and EGO when applicable), substantial disparities persist. For diffusion models on FFHQ, group gaps reach up to 3.2\% in precision and 7.7\% in recall. For LLaMA-3.2 and Gemma-3 on Wikipedia Biographies, disparities are larger, reaching up to 25.37\% in precision and 15.06\% in recall. In most cases, precision is higher for the minority group (female), whereas recall is higher for the majority group (male), suggesting a systematic trade-off across groups. Overall, these results empirically validate the brittleness of MGO/EGO: satisfying proportion-based criteria does not prevent large group-wise differences in conditional quality.

%% file: content/equalized.tex
\section{Equalized Generative Treatment (EGT)}
\label{sec:newdefinition}

The analysis in Section~\ref{sec:brittleness} shows that fairness criteria based solely on group proportions, such as MGO and EGO, are inherently brittle. To address this limitation, we introduce a simple yet fundamental fairness criterion, called \emph{equalized generative treatment} (EGT). 
We then show that applying this definition promotes the minimization of the highest $f$-divergence between conditionals among sensitive groups.

\subsection{Definition and First Property}

We now introduce the notion of \emph{equalized generative treatment} (EGT). This criterion provides a stronger, more fine-grained notion of fairness by explicitly linking generative quality to each subgroup using a reference $f$-divergence.

\begin{definition}
    Let $P, Q \in \PX$ and let $f$ be such that $\Df$ defines an $f$-divergence. For any $\delta > 0$, we say that $Q$ and $P$ satisfy \emph{$\delta$-equalized generative treatment} ($\delta$-EGT) w.r.t. $\Df$ if for all $a,a' \in \mathcal{A}$ one has
    \[
        \bigl| \Df(P_a \| Q_a) - \Df(P_{a'} \| Q_{a'}) \bigr| \leq \delta.
    \]
    When $\delta = 0$, we say that $Q$ and $P$ satisfy \emph{equalized generative treatment} w.r.t. $\Df$.
\end{definition}

This definition captures the idea that the divergence experienced by each sensitive group should be approximately equal. In contrast to MGO or EGO, which constrain only proportions, EGT enforces a parity w.r.t. the actual metric that is being used to train or evaluate our generative model.


\textbf{Conditional closure.}  Conceptually, the best way to approximate a target distribution $P$ while respecting $\delta$-EGT would be to treat each sensitive group independently. More specifically, for every $a \in \mathcal{A}$, one would learn a conditional distribution $Q_a$ that minimizes $\Df(P_a \| Q_a)$, and then recombine the conditional models using the true proportions $(\pi_a^P)_{a \in \mathcal{A}}$ of the target distribution. This procedure would ensure that each sensitive group is treated equally and that the final mixture aligns with $P$ as closely as possible. In practice, however, generative models are rarely designed in such a conditional manner. Instead, they typically produce distributions $Q$ as a whole, without the ability to freely optimize and reassemble conditionals. As a result, the typical family $\mathcal{Q}$ of candidate models used in practice seldom captures what would be achievable if sensitive group-level flexibility were available. To reason about this gap, we introduce the \emph{conditional closure} of $\mathcal{Q}$, an augmented set of distributions that allows explicit recombination of conditionals. \medskip

\begin{definition}
    Let $\mathcal{Q} \subseteq \PX$ be a family of candidate models.  For each $a \in \mathcal{A}$, let $\mathcal{Q}_a = \bigl\{ R \in \mathcal{P}_\lambda(\mathcal{X}_a) \mid \exists Q \in \mathcal{Q} \text{ with } Q_a = R \bigr\}$  be the set of conditional distributions for group $a$ within $\mathcal{Q}$.
    Let also $\Delta_\mathcal{Q} = \bigl\{ (\pi_a)_{a \in \mathcal{A}} \in \Delta(\mathcal{A}) \mid \exists Q \in \mathcal{Q} \text{ with } (\pi_a^Q)_{a \in \mathcal{A}} = (\pi_a)_{a \in \mathcal{A}} \bigr\}$ be the set of group proportions in $\mathcal{Q}$. Then the \emph{conditional closure} of $\mathcal{Q}$, denoted $\overline{\mathcal{Q}}_\mathcal{A}$, is the set of distributions for which
    is defined as the set of distributions for which the conditional w.r.t. $a$ belong to $\mathcal{Q}_a$ for ever and proportions lie in $\Delta_\mathcal{Q}$, i.e.,$
        \overline{\mathcal{Q}}_\mathcal{A} = \Bigl\{ \sum_{a \in \mathcal{A}} \pi_a R_a, \mid  (\pi_a)_{a\in\mathcal{A}} \in \Delta_\mathcal{Q} \text{ and } \forall a \in \mathcal{A}, R_a \in \mathcal{Q}_a \Bigr\}. $
\end{definition}

Intuitively, $\overline{\mathcal{Q}}_\mathcal{A}$ can be seen as the completed version of $\mathcal{Q}$ for the fairness-constrained problem. It extends $\mathcal{Q}$ by allowing to select conditional models for each sensitive group from $\mathcal{Q}$ and then reassemble them under any group proportions realizable by $\mathcal{Q}$. In this sense, $\overline{\mathcal{Q}}_\mathcal{A}$ represents a best-case modeling scenario. If the original set $\mathcal{Q}$ has been explicitly designed to support attribute-conditional subdivision, then we may have $\mathcal{Q} = \overline{\mathcal{Q}}_\mathcal{A}$. In general, however, $\overline{\mathcal{Q}}_\mathcal{A}$ strictly contains $\mathcal{Q}$, since most generative models are not conditionally structured for sensitive attributes. The role of $\overline{\mathcal{Q}}_\mathcal{A}$ is to provide a principled baseline for understanding the limits of generative modeling under fairness constraints. Theorem~\ref{thm:equalized generative treatment-lower-bound} shows that, under the $\delta$-EGT, the global divergence $\Df(P | Q)$ is bounded below by the largest conditional divergence across sensitive groups in is the best model within $\overline{\mathcal{Q}}_\mathcal{A}$.

\begin{theorem}
    \label{thm:equalized generative treatment-lower-bound}
    Let $P \in \PX$ be a non-trivial target probability distribution and $f$ be a function such that $\mathcal{D}_f$ defines an $f$-divergence. Let $\mathcal{Q}$ be set of candidate probability distributions satisfying MGO with $P$ and such that there exists $Q^\star \in \argmin_{ Q \in \overline{\mathcal{Q}}_\mathcal{A}} \Df\left( P \| Q \right).$  Then for any $\delta >0 $, if $Q \in \mathcal{Q}$ and $P$ satisfy $\delta$-EGT w.r.t. $\Df$, then
    \begin{equation*} \Df(P \| Q) \geq  \max_{a \in \mathcal{A} } \Df(P_a \| Q_a^\star) - \delta.
    \end{equation*}
\end{theorem}

This result highlights the fact that enforcing EGT promotes the minimization of the highest $f$-divergence between conditionals among groups, which naturally leads us to studying min-max schemes, as explained in the following section.

\subsection{Min-Max as Candidate for Enforcing EGT}

Before diving into empirical considerations in the next section, we first analyze the feasibility of training a model that satisfies EGT from a theoretical viewpoint. To do so, we consider an arbitrary target distribution $P$ and a set of candidate models $\mathcal{Q}$. We would like to find the distributions $Q \in \mathcal{Q}$ that approximate $P$ the best, but also need to enforce $\delta$-EGT. In this context, the most direct and natural strategy we be to solve the following regularized problem
\[
    \min_{Q \in \mathcal{Q}} \Df(P \| Q) +  \lambda \hspace{-5pt} \sum_{a,a'\in\mathcal{A}}
    \bigl| \Df(P_a \| Q_a) - \Df(P_{a'} \| Q_{a'}) \bigr|,
\]
where $\lambda \geq 0$ controls the strength of the EGT regularization. $\lambda=0$ recovers standard divergence minimization, while larger values place increasing emphasis on balancing the conditional divergences across groups. However, while conceptually appealing, directly enforcing EGT via $f$-divergence regularization is impractical in modern training loops. $f$-divergences are inherently one-sided, yielding either a tractable minimization objective (variational) or a maximization surrogate (adversarial), but not both simultaneously~\citep{huszar_how_2015,arjovsky_towards_2017}.

To circumvent the limitations of directly enforcing EGT through regularization, the next natural candidate is to rewrite the problem as a min-max optimization scheme,
\begin{equation*}
    \min_{Q\in\mathcal{Q}}\max_{a\in\mathcal{A}} \Df\left(P_a\Vert  Q_a\right).\label{eq:minmax}
\end{equation*}

Instead of minimizing a regularized $f$-divergence, this novel objective minimizes the worst conditional divergence across groups, pushing the model toward balancing its errors. This method thus directly aligns with the theoretical characterization of EGT given in Theorem~4.3. Contrarily to regularization, min-max schemes are practical as they are already implemented in many context related to distributionally robust optimization and group-robust learning~\cite{oren_distributionally_2019,sagawa_distributionally_2020}. In the next section, we therefore study this kind of scheme as a practical approximations to EGT rather than relying on direct regularization.

%% file: content/experiments.tex
\begin{figure*}[!bt]
    \centering
    \includegraphics[width=0.95\textwidth]{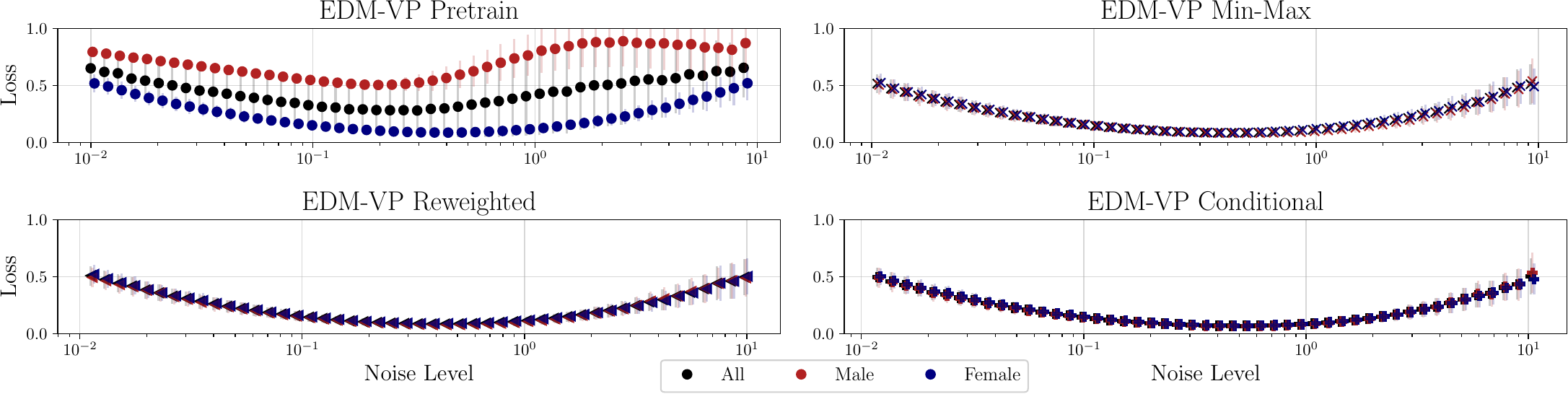}
    \caption{Estimated denoising losses per noise level for EDM-VP trained on FFHQ. Baseline exhibits a persistent gap between male and female groups. Reweighting and Min–Max reduce the gap, while conditional training almost eliminates it.}
    \label{fig:ffhq_losses_comparison}
\end{figure*}

\section{Improved fairness through EGT}
\label{sec:improved}

This section evaluates whether optimization strategies that balance training losses across sensitive groups also improve fairness under our EGT criterion on evaluation metrics.

\textbf{Compared methods.}
Three strategies are considered.
\textit{Min--Max training} is the only approach that explicitly targets EGT by minimizing the worst group-wise conditional divergence (Theorem~4.3). This popular objective where equalizing (or protecting) worst-group \emph{training} losses is often observed to improve robustness. In generative modeling, however, such equalization is typically reported at the level of the optimized loss, and its impact on sample-based evaluation fairness metrics is less understood. \textit{Conditional training} provides direct control of MGO/EGO at generation time by learning attribute-conditional distributions. While it does not optimize EGT explicitly, it mitigates the brittleness highlighted by Theorem~3.1 by preventing the training objective from collapsing onto a single group. In this sense, conditioning can be seen as a practical mechanism to avoid pathological error concentration, even though it does not enforce EGT. Finally, \textit{reweighted loss} is a widely used baseline in fair generative modeling: it rescales per-sample losses to emphasize underrepresented groups and is primarily designed to improve the proportion-based criteria EGO. Importantly, RW comes with no guarantee for EGT and can even \emph{increase} conditional quality gaps.

\textbf{Methods and implementation details.} Since the model families and datasets are already introduced in Section~\ref{sec:brittleness} and fully detailed in Appendix~\ref{app:sec:experiments}, the focus here is on how the three training strategies are implemented. For diffusion, the baseline starts from the public EDM checkpoints (VP/VE) of \citet{karras_elucidating_2022}. For text, LoRA adapters are fine-tuned on biography generation for LLaMA-3.2-Chat (1B/3B) and Gemma-3 (4B, pretrained and instruction-tuned). Conditional variants are obtained by adding standard class-conditioning to diffusion and by training LLMs under an attribute-constrained prompt (e.g., ``a male/female person''). Reweighting rescales the negative log-likelihood (LLMs) or denoising/reconstruction objective (diffusion) with group-dependent weights  equal to $\pi_a^P/\pi_a^Q$.

\textbf{Min--Max training and stability improvements.}
For diffusion, Min--Max must account for the multi-noise nature of training: rather than selecting a single worst group globally, the max is taken \emph{per noise level} (Algorithm~\ref{alg:minmax-diffusion-noise}), ensuring that no subset of the diffusion trajectory disproportionately harms a group. This ``per-noise'' maximization is essential in practice because different groups can dominate the loss at different noise regimes. For both diffusion and LLM fine-tuning, the worst-group selection is stabilized via an exponential moving average (EMA) of group-wise losses (Algorithm~\ref{alg:minmax-ema}). This EMA acts as a short memory that reduces batch-to-batch variance, prevents rapid oscillations of the maximizer, and avoids degenerate updates driven by a single noisy minibatch. Empirically, it yields markedly more stable optimization than naive per-batch maximization, especially in the low-batch regime of LLM fine-tuning.

\textbf{Image generation: closing loss gaps is not enough.}
For EDM-VP on FFHQ, Figure~\ref{fig:ffhq_losses_comparison} reports estimated denoising losses across noise levels. The pretrained baseline exhibits persistent loss gaps between male and female groups, and all three strategies substantially reduce these gaps (average losses are reported in Appendix Table~\ref{tab:train_loss_uncond_cond}). This observation matches prior evidence in robust learning: Min-Max objectives can effectively \emph{equalize the optimized training loss} across groups. However, balancing the training loss does not guarantee balanced performance on evaluation metrics. The training objective is only a surrogate for generative quality, and improvements in loss can translate imperfectly to sample-based evaluation metrics. Therefore, closing subgroup loss gaps should be viewed as a useful diagnostic, but not a sufficient condition for improved EGT at evaluation.

\textbf{Image generation: improvements in EGT metrics.}
Table~\ref{tab:ffhq_all_deltas_bold_nested_nostd_pr_vp_all} summarizes EGT disparities on precision, recall, and their sum $P\!+\!R$ (a valid $f$-divergence by linearity), together with $\delta$-FID for completeness. On EDM-VP, Min--Max yields a clear and stable gain: $\delta$-PR drops from $2.2$ (pretrained) to $0.3$, with $\delta$-P and $\delta$-R collapsing from $2.0/4.2$ to $0.1/0.2$. In contrast, conditional training can reduce $\delta$-P but may worsen $\delta$-R (and therefore $\delta$-PR), and reweighting sometimes improves disparities but not consistently across settings. Overall, explicitly optimizing the worst-group divergence appears necessary to obtain systematic improvements under EGT.

\begin{table}[!t]
    \centering
    \caption{Evaluation of the different optimization methods on the FFHQ dataset with EDM models. We report the difference in precision/recall  between subgroups. All results are in percentage points. The best results are in bold per method generation method and per metric and disparties better than the unconstrained baseline are underlined.}
    \label{tab:ffhq_all_deltas_bold_nested_nostd_pr_vp_all}
    \input{tables/ffhq_vpve.tex}
    \vspace{-10pt}
\end{table}

\textbf{Text generation: larger initial gaps and partial improvements.}
Results on Wikipedia Biographies (Table~\ref{tab:llama_joint_delta_topo_with_mgoego_withstd_all}) exhibit substantially larger initial disparities than diffusion models. As in the image setting, conditional training and reweighting yield mixed outcomes: improvements are sometimes observed, but they are not systematic. Min--Max again provides the most consistent reduction in $\delta$-P, $\delta$-R, and $\delta$-PR, although the magnitude of improvement is smaller in absolute terms. A key practical difference is computational: LLM fine-tuning is significantly more expensive than diffusion fine-tuning in our setup, which limits the extent to which optimization can reshape representations.

\begin{table}[!t]
    \caption{Evaluation of the different optimization methods used to finetune LLaMa3.2 Chat 1b and 3b on the Wikipedia Biographies dataset. We report the difference in precision/recall  between subgroups. All results are in percentage points.  The best results are in bold per method generation method and per metric and disparties better than the unconstrained baseline are underlined.}
    \label{tab:llama_joint_delta_topo_with_mgoego_withstd_all}
    \centering
    \small
    \input{tables/llms.tex}
    \vspace{-10pt}
\end{table}

\textbf{Takeaway.}
Overall, balancing group-wise training losses is not sufficient to guarantee EGT improvements on evaluation divergences. Among the tested strategies, Min--Max is the only method that \emph{systematically} reduces EGT disparities across modalities, making it a promising approach for fairness beyond proportion matching. In addition, the empirical correlations in Table~\ref{tab:correlations_all_deltas} computed across all experiments confirm the brittleness result from Section~\ref{sec:brittleness}: $\delta$-MGO and $\delta$-EGO show only negligible association with $\delta$-P, $\delta$-R, and $\delta$-PR, so matching proportions alone provides little predictive control over group-wise conditional quality.
\begin{table}[H]
    \centering
    \small
    \setlength{\tabcolsep}{3.5pt}
    \caption{Correlations between different $\delta$-MGO, $\delta$-EGO, $\delta$-Precision ($\delta$-P), $\delta$-Recall ($\delta$-R), $\delta$-Precision/Recall ($\delta$-PR), and $\delta$-FID, computed over all diffusion and LLM experiments.}
    \input{tables/correlations2.tex}
    \label{tab:correlations_all_deltas}
\end{table}

\textit{Disclaimer.} These fairness improvements do not come for free. As discussed in Section~\ref{sec:newdefinition}, reducing disparities typically comes at the expense of the best-performing group. This trade-off is visible in Appendix~\ref{app:sec:experiments}. For instance, under MGO with EDM-VP, Min--Max reduces $\delta$-PR from $2.2$ to $0.3$, but recall for the strongest subgroup drops from $91.7$ to $86.6$. Similarly, conditional training achieves strong $\delta$-FID (down to $0.08$) but can substantially increase the overall FID (from $2.34$ to $2.75$). Overall, EGT gains should therefore be interpreted as a redistribution of errors across groups rather than a uniform improvement for every subgroup.

%% file: tables/ffhq_vpve.tex
\begin{tabular}{l l r r r r}
\toprule
Model & Method & $\delta$-P & $\delta$-R & $\delta$-PR & $\delta$-FID \\
\midrule
\multirow{4}{*}{VP} & Pretrain & 2.0 & 4.2 & 2.2 & 0.34 \\
 & Conditional & \textbf{1.6} & 6.8 & 5.2 & \underline{\textbf{0.08}} \\
 & Reweighted & \textbf{1.8} & \textbf{2.3} & \textbf{0.6} & 0.44 \\
 & Min-Max & \underline{\textbf{0.1}} & \underline{\textbf{0.2}} & \underline{\textbf{0.3}} & \textbf{0.24} \\
\midrule
\multirow{4}{*}{VE} & Pretrain & 2.1 & 4.4 & 2.3 & 0.35 \\
 & Conditional & \textbf{0.9} & \textbf{2.7} & 3.5 & \textbf{0.11} \\
 & Reweighted & \textbf{0.8} & \textbf{3.0} & \textbf{2.2} & \underline{\textbf{0.04}} \\
 & Min-Max & \underline{\textbf{0.3}} & \underline{\textbf{1.1}} & \underline{\textbf{0.8}} & \textbf{0.33} \\
\bottomrule
\end{tabular}

%% file: tables/llms.tex
\begin{tabular}{l l r r r}
\toprule
Model & Method & $\delta$-P & $\delta$-R & $\delta$-PR \\
\midrule
\multirow{4}{*}{LLaMA-3.2-Chat 1B} & Pretrained & 21.18 & 11.46 & 32.65 \\
 & Conditional & \underline{\textbf{19.09}} & \textbf{7.81} & \underline{\textbf{26.90}} \\
 & RW & 21.29 & \textbf{9.22} & \textbf{30.50} \\
 & MinMax & \textbf{19.55} & \underline{\textbf{7.73}} & \textbf{27.28} \\
\midrule
\multirow{4}{*}{LLaMA-3.2-Chat 3B} & Pretrained & 25.37 & 15.06 & 40.43 \\
 & Conditional & \textbf{24.31} & \underline{\textbf{12.26}} & \textbf{36.57} \\
 & RW & \textbf{22.88} & \textbf{12.29} & \textbf{35.16} \\
 & MinMax & \underline{\textbf{22.27}} & \textbf{12.87} & \underline{\textbf{35.13}} \\
\midrule
\multirow{4}{*}{Gemma-3 4B (pt)} & Pretrained & 20.90 & 12.13 & 33.03 \\
 & Conditional & 23.17 & \underline{\textbf{11.08}} & 34.25 \\
 & RW & \textbf{19.67} & \textbf{11.71} & \textbf{31.38} \\
 & MinMax & \underline{\textbf{19.13}} & \textbf{11.67} & \underline{\textbf{30.79}} \\
\midrule
\multirow{4}{*}{Gemma-3 4B (it)} & Pretrained & 22.30 & 6.50 & 28.81 \\
 & Conditional & 24.16 & 8.41 & 32.57 \\
 & RW & \underline{\textbf{20.91}} & \underline{\textbf{6.20}} & \underline{\textbf{27.11}} \\
 & MinMax & \textbf{21.65} & \textbf{6.28} & \textbf{27.93} \\
\bottomrule
\end{tabular}

%% file: tables/correlations2.tex
\begin{tabular}{l r r r r r r r}
    \toprule
                 & $\delta$-MGO & $\delta$-EGO & $\delta$-P & $\delta$-R & $\delta$-PR & $\delta$-FID \\
    \midrule
    $\delta$-MGO & 1.00         & -1.00        & 0.04       & -0.10      & -0.04       & -0.20        \\
    $\delta$-EGO & -1.00        & 1.00         & -0.18      & -0.08      & -0.14       & 0.20         \\
    \bottomrule
\end{tabular}

%% file: content/conclusion.tex
\section{Concluding Remarks}


In this work, we demonstrated that existing proportion-based criteria for fairness in generative modeling are inherently brittle: even when perfectly satisfied, a model’s output quality can remain arbitrarily unbalanced across sensitive groups. To address this, we introduced equalized generative treatment (EGT), a fairness definition that enforces comparable $f$-divergences across groups, and studied its applicability. While $f$-divergences are among the most widely used metrics in generative models, alternatives such as the FID or integral probability metrics also exist and merit further investigation. In particular, it would be interesting to examine whether our analysis (especially results analogous to Theorems~\ref{th:failure} and~\ref{th:lowerbound}) remains valid in these alternative settings. Similarly, we expect that our results could extend to training schemes that directly minimize Wasserstein distances, such as Wasserstein GANs. We leave these explorations to future work.


%% file: content/appendix/statement.tex


%% file: content/appendix/mathsupp.tex
\section{Mathematical Supplementary Material}

\subsection{Useful Lemma on the Surjectivity of $\Df$ }

We begin with a statement that is of independent interest, as it establishes the surjectivity of the mapping $\mathcal{D}_f(R, \cdot)$ under suitable continuity assumptions on $f$. This result is stated in Lemma~\ref{lemma:surjectivity}. \medskip

\begin{lemma}
    \label{lemma:surjectivity}
    Let $f : [0, +\infty) \to ( - \infty, + \infty ]$ be a continuous function such that $\mathcal{D}_f$ defines an $f$-divergence. Then for any $R \in \PX$, the map $\mathcal{D}_f(R, \cdot) : \{ Q \in \PX \mid \Supp(R)=\Supp(Q) \} \to \R_+$ is surjective onto the interval $(0, f(0) + \bar{f}(+\infty))$.
\end{lemma}

\begin{proof}
    Let $R \in \PX$, and let $f: [0, +\infty) \to ( - \infty, + \infty ] $ be a convex and continuous function with $f(1) = 0$. Let us also fix $\alpha \in (0,1)$ and $\beta \in (1,+\infty)$. Since $R$ is absolutely continuous with respect to the Lebesgue measure $\lambda$, its cumulative density function is also continuous on $\mathbb{R}$. Accordingly, by definition of the cumulative density function and by the intermediate value theorem, there exists $A_\alpha \in \mathcal{B}(\mathcal{X})$ such that $R(A_\alpha) = \alpha$. \medskip

    \emph{1) Construction of an auxiliary mapping $\phi$.} Let us denote by $r:=\frac{d R}{d\lambda}$ the probability density function of $R$ with respect to $\lambda$. Thanks to the above, we can define $Q_\alpha^\beta$ the probability distribution in $\PX$ that admits a probability density function $\frac{d Q_\alpha^\beta}{d\lambda} = q_\alpha^\beta$ defined for all $x \in \mathcal{X}$ as

    \begin{equation}
        q_\alpha^\beta(x) := \frac{1}{\beta} \left( \frac{1}{\alpha} r(x) \mathds{1}_{A_\alpha}(x) \right) + \left(1 - \frac{1}{\beta} \right) \left( \frac{1}{1 - \alpha} ~r(x) \mathds{1}_{\mathcal{X} \setminus A_\alpha}(x) \right). \label{eq:construction}
    \end{equation}

    By construction, since $\beta >1$ and $\alpha \in (0,1)$, $q_\alpha^\beta$ is a valid probability density function and $Q_\alpha^\beta \in \PX$. Furthermore, also by construction we have $\Supp(R)=\Supp(Q_\alpha^\beta)$ (see Appendix~\ref{sec:well-definiteness} for more details).
    Hence, using the alternative characterization of $f$-divergences in the special case of matching support (see e.g.~\citep[Chapter 7]{polyanskiy_information_2025}), we have
    \begin{align*}
        \mathcal{D}_f(R \| Q_\alpha^\beta) & = \int_{\mathcal{X}} f\left(\frac{r(x)}{q_\alpha^\beta(x)}\right) q_\alpha^\beta(x) ~d \lambda(x) \,\,\,\, \text{  with the convention $f(\tfrac{0}{0})\times0=0$.}
    \end{align*}
    Using the above, and by definition of $Q_\alpha^\beta$, we thus have
    \begin{align*}
        \mathcal{D}_f(R \| Q_\alpha^\beta) & = \int_{A_\alpha} f\left( \frac{r(x)}{\tfrac{r(x)}{\beta \alpha}} \right) \frac{r(x)}{\alpha \beta} ~d \lambda(x) + \int_{\mathcal{X} \setminus A_\alpha} f\left(\frac{r(x)}{r(x)\tfrac{\beta - 1}{(1-\alpha) \beta} }\right) \frac{\beta - 1}{\beta (1 - \alpha)} r(x) ~d \lambda(x) \\[2pt]
                                           & = \int_{A_\alpha} \frac{f\left(\beta \alpha\right) }{\alpha \beta} r(x) ~d \lambda(x) + \int_{\mathcal{X} \setminus A_\alpha} f\left(\frac{(1-\alpha) \beta}{\beta - 1}\right) \frac{\beta - 1}{\beta (1 - \alpha)} r(x) ~d \lambda(x).
    \end{align*}
    Furthermore, by linearity of the integral and by construction of $A_\alpha$, we have
    \begin{align}
        \mathcal{D}_f(R \| Q_\alpha^\beta)
         & =  \frac{1}{\alpha \beta} f\left(\beta \alpha\right) \int_{A_\alpha} r(x) ~d \lambda(x) +  \frac{\beta - 1}{\beta (1 - \alpha)} f\left(\frac{(1-\alpha) \beta}{\beta - 1}\right) \int_{\mathcal{X} \setminus A_\alpha} r(x) ~d \lambda(x) \notag \\[2pt]
         & = \frac{1}{\alpha \beta} f\left(\beta \alpha\right)  R\left( A_\alpha \right) + \frac{\beta - 1}{\beta (1 - \alpha)} f\left(\frac{(1-\alpha) \beta}{\beta - 1}\right) R \left(\mathcal{X} \setminus A_\alpha \right) \notag                      \\[2pt]
         & = \frac{1}{\beta} f\left(\beta \alpha\right)  +  \frac{\beta - 1}{\beta} f\left(\frac{(1-\alpha) \beta}{\beta - 1}\right). \label{eq:psi}
    \end{align}
    Since the $\alpha$, and $ \beta$ have been chosen arbitrarily. The above construction holds for any $(\alpha, \beta) \in (0,1)\times(1, +\infty)$. Thus, we can define $\phi: (0,1) \times (1, +\infty) \rightarrow \mathbb{R}_+$ the mapping such that \[ \phi(\alpha, \beta) := \mathcal{D}_f(R \| Q_\alpha^\beta) = \frac{1}{\beta} f\left(\beta \alpha\right)  +  \frac{\beta - 1}{\beta} f\left(\frac{(1-\alpha) \beta}{\beta - 1}\right),~ \forall (\alpha, \beta ) \in (0,1) \times (1, +\infty). \]

    As $\{ Q_\alpha^\beta \mid (\alpha, \beta) \in (0,1)\times(1, +\infty) \} \subseteq \{ Q \in \PX \mid \Supp(R) = \Supp(Q) \}$, to get the expected result it suffices to show that $\phi$ is surjective onto the interval $(0, f(0) + \bar{f}(+\infty))$.

    \emph{2) Studying the surjectivity of $\phi$.} As $f$ is a continuous function, by composition of continuous functions, $\phi$ is jointly continuous on $(0,1) \times (1, +\infty)$. Accordingly, still using the intermediate value theorem, $\phi$ is surjective onto $[\phi(\alpha_a,\beta_a),\phi(\alpha_b,\beta_b)]$ for any $\alpha_a,\alpha_b \in (0,1)$ and $\beta_a, \beta_b \in (1,+\infty)$. In particular, since this holds for any choice of $\alpha_a,\alpha_b,\beta_a,$ and $\beta_b$ we also have that $\phi$ is surjective onto $[\phi(1/2, 2), \lim_{\substack{\alpha \rightarrow 1 \\ \beta \rightarrow +\infty}}\phi(\alpha, \beta) )$. To conclude, we just need to compute each of these terms: \vspace{-5pt}
    \begin{itemize}
        \item $\phi(1/2, 2) = \frac{1}{2} f(1) + \frac{1}{2} f(1) = 0$, and
        \item $\lim\limits_{\substack{\alpha \rightarrow 1 \\ \beta \rightarrow +\infty}}\phi(\alpha, \beta) = \lim\limits_{\beta \rightarrow \infty} \frac{1}{\beta} f(\beta) + \lim\limits_{\beta \rightarrow +\infty} \frac{\beta -1}{\beta} f(0) = \bar{f}(+\infty) + f(0)$.
    \end{itemize}

    \emph{3) Conclusion.} By surjectivity of $\phi$ and by construction of $\{ Q_\alpha^\beta \mid (\alpha, \beta) \in (0,1)\times(1, +\infty) \}$, we just showed that for any $y \in (0,f(0)+\bar{f}(\infty))$, there exists $Q \in \{ Q \in \PX \mid \Supp(R)=\Supp(Q) \}$ such that $\Df(R,Q)=y$, which concludes the proof.
\end{proof}

\subsubsection{Additional sanity checks for the construction of $Q_\alpha^\beta$ in Lemma~\ref{lemma:surjectivity}}

\label{sec:well-definiteness}

\paragraph{Well-definiteness of the distribution.} Let us first show that for any $(\alpha, \beta) \in (0,1)\times(1,+\infty)$, $q^{\beta}_\alpha$ is a valid probability density function. For this, first note that the terms $1/\beta$, $1- 1/\beta$, $ 1/\alpha$ and $ 1/(1-\alpha)$ are all positive. Furthermore, $r$ is itself a probability density function by definition, hence non-negative. Hence, by construction $q^{\beta}_\alpha$ is a non-negative function. Also note that, for any $(\alpha, \beta) \in (0,1)\times(1,+\infty)$,  integrating against $\lambda$ over $\mathcal{X}$, we get

\begin{align*}
    \int_{\mathcal{X}} q_\alpha^\beta(x)~  d\lambda(x) & = \int_{\mathcal{X}} \frac{1}{\beta} \left( \frac{1}{\alpha} r(x) \mathds{1}_{A_\alpha}(x) \right) + \left(1 - \frac{1}{\beta} \right) \left( \frac{1}{1 - \alpha} ~r(x) \mathds{1}_{\mathcal{X} \setminus A_\alpha}(x) \right) d\lambda(x) \\[2pt]
                                                       & = \int_{A_\alpha} \frac{1}{\beta \alpha} r(x) d \lambda(x)+ \int_{\mathcal{X} \setminus A_\alpha}\left(1 - \frac{1}{\beta} \right) \frac{1}{1 - \alpha} ~r(x) d\lambda(x)                                                                   \\
    \intertext{Which by linearity of the integral and definition of $r$ and $A_\alpha$ gives}
                                                       & = \frac{1}{\beta \alpha} R(A_\alpha) + \left(1 - \frac{1}{\beta} \right) \frac{1}{1 - \alpha} R(\mathcal{X} \setminus A_\alpha) = \frac{1}{\beta } + \left(1 - \frac{1}{\beta} \right) = 1.
\end{align*}

\paragraph{Matching supports.} Let us now show that $\Supp(Q_\alpha^\beta) = \Supp(R)$. To do so let us first consider $x \notin \Supp(R)$, by definition of $q_\alpha^\beta$ we have
\[ q_\alpha^\beta(x) = \frac{1}{\beta \alpha} r(x)\mathds{1}_{A_\alpha}(x)  + \left(1 - \frac{1}{\beta} \right) \frac{1}{1 - \alpha} ~r(x) \mathds{1}_{\mathcal{X} \setminus A_\alpha}(x).\]
Since $x \notin \Supp(R)$, we have $r(x)=0$, hence $ q_\alpha^\beta(x) = 0$. Accordingly, $x \notin \Supp(Q_\alpha^\beta)$. This means by contrapositive that  $\Supp(Q_\alpha^\beta) \subset \Supp(R)$. Similarly, let $x \notin \Supp(Q_\alpha^\beta)$ we have
\[ q_\alpha^\beta(x) = \frac{1}{\beta \alpha} r(x) \mathds{1}_{A_\alpha}(x)  + \left(1 - \frac{1}{\beta} \right) \frac{1}{1 - \alpha} ~r(x) \mathds{1}_{\mathcal{X} \setminus A_\alpha}(x) = 0.\]
Since all terms $1/\beta$, $1- 1/\beta$, $ 1/\alpha$ and $ 1/(1-\alpha)$ are positive, this means that $r(x) = 0$. Hence $x \notin \Supp(R)$, which gives us $\Supp(R) \subset \Supp(Q_\alpha^\beta)$.



\subsection{Proof of Theorem 1}

Before proceeding to the proof of Theorem~1, we state a central lemma that decomposes the $f$-divergence between two measures $P$ and $Q$ in terms of a linear combination of the $f$-divergence between their marginals. The result is given in Lemma~\ref{lemma:decmpositiondf}

\subsubsection{Preliminary Lemma}

\begin{lemma}
    \label{lemma:decmpositiondf}
    Let $P\in \PX$ be a non-trivial target probability distribution, $Q \in \PX$, and $f$ be function such that $\mathcal{D}_f$ defines an $f$-divergence. If $P$ and $Q$ have matching generative odds, then
    \[ \Df(P \Vert Q) =  \sum_{a \in \mathcal{A}} \pi^Q_a \Df(P_a \Vert Q_a), \] where the decomposition according to $\mathcal{A}$ is as defined in Section~\ref{sec:fairness}.
\end{lemma}

\begin{proof} Let $P \in \PX$ be a target probability distribution and $f$ be a function such that $\mathcal{D}_f$ defines an $f$-divergence. Let $Q \in \PX$, such that $P$ and $Q$ have matching generative odds. By definition of the attribute mapping $\psi$, $\{ \mathcal{X}_a \mid a \in \mathcal{A} \}$ is a partition of $\mathcal{X}$. Hence, denoting $p = \frac{dP}{d\lambda}$ and $q = \frac{dQ}{d\lambda}$ the respective probability density functions of $P$ and $Q$, we have
    \begin{align*}
        \Df(P \Vert Q) & = \int_{\Supp(Q)} f\left(\frac{p(x)}{q(x)}\right) q(x) d \lambda(x) + \bar{f}(+\infty) P\left(\mathcal{X} \setminus \Supp(Q) \right)                                               \\[2pt]
                       & = \sum_{a \in \mathcal{A}} \int_{\mathcal{X}_a \cap \Supp(Q)} f\left(\frac{p(x)}{q(x)} \right) q(x) d \lambda(x) + \bar{f}(+\infty) P\left(\mathcal{X} \setminus \Supp(Q) \right).
    \end{align*}

    Note that, for any $a \in \mathcal{A}$, by definition we have
    \begin{align*}
        \mathcal{X}_a \cap \Supp(Q) := \{ x \in \mathcal{X}_a \mid q(x) > 0\}
        = \{ x \in \mathcal{X} \mid q(x)\times \mathds{1}_{\mathcal{X}_a}(x) > 0\}.
    \end{align*}
    Now recall that we assumed $\pi^P_a > 0$ for any $a \in \mathcal{A}$. Furthermore, $Q$ is assumed to match the odds of $P$, meaning that $\pi^Q_a = \pi^P_a > 0$ for all $a \in \mathcal{A}$. Hence we have
    \begin{align}
        \mathcal{X}_a \cap \Supp(Q) & = \{ x \in \mathcal{X} \mid q(x)\times \mathds{1}_{\mathcal{X}_a}(x) > 0\} \nonumber                                                           \\
                                    & = \{ x \in \mathcal{X} \mid \tfrac{q(x)}{\pi_a^Q}\times \mathds{1}_{\mathcal{X}_a}(x) := q_a(x) > 0\} =\Supp(Q_a).  \label{eq:Supportequality}
    \end{align}
    Using~(\ref{eq:Supportequality}) in the first decomposition of $\Df$, we get
    \begin{align*}
        \Df(P \Vert Q) = \sum_{a \in \mathcal{A}} \int_{\Supp(Q_a)} f\left(\frac{p(x)}{q(x)} \right) q(x) d \lambda(x) + \bar{f}(+\infty) P\left(\mathcal{X} \setminus \Supp(Q) \right).
    \end{align*}
    Now recall that we can always rewrite $p$ and $q$ as mixtures of attribute-conditional probability density functions $p = \sum_{a \in \mathcal{A}} \pi_a^P p_a$ and $q = \sum_{a \in \mathcal{A}} \pi_a^Q q_a$ where for any $a \in \mathcal{A}$, $q_a(x) = p_a(x) = 0$ for all $x \notin \mathcal{X}_a$. Accordingly, the $f$-divergence between $P$ and $Q$ can be rewritten as
    \begin{align}
        \label{eq:firstdecompDF}  \Df(P \Vert Q) = \sum_{a \in \mathcal{A}} \int_{\Supp(Q_a)} f\left(\frac{\pi_a^P p_a(x)}{\pi_a^Q q_a(x)} \right) \pi_a^Q q_a(x) d \lambda(x) + \bar{f}(+\infty) P\left(\mathcal{X} \setminus \Supp(Q) \right).
    \end{align}

    Note also that, by definition of the by definition of the attribute oracle, we have $P = \sum_{a \in \mathcal{A}} \pi_a^P P_a$ where $\Supp(P_a) \subset \mathcal{X}_a$. Using this we can rewrite $P\left(\mathcal{X} \setminus \Supp(Q) \right)$ as
    \begin{align*}
        P\left(\mathcal{X} \setminus \Supp(Q) \right) & = \sum_{a \in \mathcal{A}} \pi^P_a P_a(\mathcal{X} \setminus \Supp(Q)) = \sum_{a \in \mathcal{A}} \pi^P_a P_a(\mathcal{X}_a \setminus \Supp(Q)) \\
                                                      & = \sum_{a \in \mathcal{A}} \pi^P_a P_a(\mathcal{X}_a \setminus \left( \Supp(Q) \cap  \mathcal{X}_a \right) )                                    \\
                                                      & = \sum_{a \in \mathcal{A}} \pi^P_a P_a(\mathcal{X}_a \setminus \Supp(Q_a) ).
    \end{align*}
    Where the last line comes from using (\ref{eq:Supportequality}). Using the above in (\ref{eq:firstdecompDF}) we get
    \begin{align*}
        \Df(P \Vert Q) & = \sum_{a \in \mathcal{A}} \int_{\Supp(Q_a)} f\left(\frac{\pi_a^P p_a(x)}{\pi_a^Q q_a(x)} \right) \pi_a^Q q_a(x) d \lambda(x) + \bar{f}(+\infty) \sum_{a \in \mathcal{A}} \pi^P_a P_a\left( \mathcal{X}_a \setminus \Supp(Q_a) \right).
    \end{align*}
    Finally, Using the matching odds property ,i.e., the fact that $ \pi^P_a= \pi_a^Q$ for all $a \in \mathcal{A}$, we have
    \begin{align*}
        \Df(P \Vert Q) & = \sum_{a \in \mathcal{A}} \int_{\Supp(Q_a)} f\left(\frac{p_a(x)}{ q_a(x)} \right) \pi_a^Q q_a(x) d \lambda(x) + \bar{f}(+\infty) \sum_{a \in \mathcal{A}} \pi^Q_a P_a \left(\mathcal{X}_a \setminus  \Supp(Q_a) \right) \\[2pt]
                       & = \sum_{a \in \mathcal{A}} \pi^Q_a \left( \int_{\Supp(Q_a)} f\left(\frac{p_a(x)}{ q_a(x)} \right) q_a(x) d \lambda(x) + \bar{f}(+\infty) P_a\left(\mathcal{X}_a \setminus  \Supp(Q_a) \right) \right)                    \\
                       & =  \sum_{a \in \mathcal{A}} \pi^Q_a \Df(P_a \Vert Q_a), \text{  which concludes the proof.}
    \end{align*}

\end{proof}

\subsubsection{Proof of the Theorem}

We now turn to the proof of Theorem~1, restated below as Theorem~\ref{th:failure}. The argument relies primarily on Lemma~\ref{lemma:surjectivity} and Lemma~\ref{lemma:decmpositiondf}.

\begin{theorem}
    \label{th:failure}
    Let $P \in \PX$ be a target probability distribution satisfying equalized generative odds and $f$ be a continuous function such that $\mathcal{D}_f$ defines an $f$-divergence. For any $\epsilon \in (0, f(0) + \bar{f}(+\infty))$ and $\gamma \in (0,\epsilon)$, there exists $Q^{\gamma} \in \mathcal{S}_{\Df}(P,\epsilon)$ satisfying matching generative odds with $P$, but for which there exists $\bar{a} \in \mathcal{A}$ such that
    \[
        \mathcal{D}_f(P_{\bar{a}} \| Q^{\gamma}_{\bar{a}}) \geq \mathcal{D}_f(P_{a} \| Q^{\gamma}_{a}) + \gamma,~ \text{ for all } a \in \mathcal{A} \setminus \{ \bar{a}\}.
    \]
\end{theorem}

\begin{proof}
    Let $P \in \PX$ be a target probability distribution satisfying equalized generative odds and $f$ be a continuous function such that $\mathcal{D}_f$ defines an $f$-divergence. By definition of the attribute mapping $\psi$, $\{ \mathcal{X}_a \mid a \in \mathcal{A} \}$ is a partition of $\mathcal{X}$. Hence we can rewrite $P = \sum_{a \in \mathcal{A}} \pi_a^P P_a$ with $\Supp(P_a) \subset \mathcal{X}_a$. Let us now fix $\epsilon \in (0, f(0) + \bar{f}(+\infty))$, $\gamma \in (0,\epsilon)$ and let $\bar{a} \in \mathcal{A}$. Since $f$ is continuous, by Lemma~\ref{lemma:surjectivity}, we know there exists $Q_{\bar{a}}^\gamma \in \PX$ such that $\Supp(P_{\bar{a}})=\Supp(Q_{\bar{a}}^\gamma)$ and $\Df(P_{\bar{a}} \Vert Q_{\bar{a}}^\gamma) = \epsilon + \gamma \frac{ \vert \mathcal{A} \vert - 1}{\vert \mathcal{A} \vert}$. Similarly, for any $a \in \mathcal{A}\setminus \{ \bar{a}\}$ there exist a distribution $Q_a^\gamma \in \PX$ such that $\Supp(P_{a})=\Supp(Q_{a}^\gamma)$ and $\Df(P_{a} \Vert Q_{a}^\gamma) = \epsilon - \frac{\gamma}{\vert \mathcal{A} \vert}$. Using these, we define $Q^\gamma$ as \[Q^\gamma = \frac{1}{\vert \mathcal{A} \vert} \sum\limits_{a \in \mathcal{A}} Q_a^\gamma.\]



    Let us now consider $Q^\gamma$ as defined above. Recall that the target distribution $P$ is assumed to satisfy equalized odds, hence $\pi^{P}_a = \tfrac{1}{\vert \mathcal{A} \vert}$ for all $a \in \mathcal{A}$. Furthermore, by definition, $\pi^{Q^\gamma}_a = \tfrac{1}{\vert \mathcal{A} \vert} = \pi^{P}_a$ for all $a \in \mathcal{A}$, which means that $Q^\gamma$ satisfies matching odds with $P$. Hence, using Lemma~\ref{lemma:decmpositiondf}, we have
    \[ \Df(P \Vert Q^\gamma) =  \sum_{a \in \mathcal{A}} \pi^{Q_a^\gamma} \Df(P_a \Vert Q_a^\gamma) = \frac{1}{\vert \mathcal{A} \vert} \sum_{a \in \mathcal{A}} \Df(P_a \Vert Q_a^\gamma) = \epsilon \]
    Nevertheless, by construction, we also have \[ \Df(P_{\bar{a}} \Vert Q_{\bar{a}}) \geq \Df(P_{a} \Vert Q_{a}) + \gamma. \]
    Hence, by construction, we just designed a probability distribution $Q^\gamma$ satisfying the conditions of Theorem~\ref{th:failure}.
\end{proof}

\subsection{Proof of Theorem 2}

We now present the proof of Theorem~2, restated below as Theorem~\ref{th:lowerbound}. The proof relies primarily on Lemma~\ref{lemma:decmpositiondf} and the notion of conditional closure. 

\begin{theorem}
\label{th:lowerbound}
    Let $P \in \PX$ be a target probability distribution and $f$ be a function such that $\mathcal{D}_f$ defines an $f$-divergence. Let $\mathcal{Q}$ be set of candidate probability distributions satisfying matching odds with respect to $P$ and such that there exists $Q^\star \in \argmin_{ Q \in \overline{\mathcal{Q}}_\mathcal{A}} \Df\left( P \| Q \right).$  Then for any $\delta >0 $, if $Q \in \mathcal{Q}$ satisfies $\delta$-equalized generative treatment of $P$ with respect to $\Df$, then
    \[ \Df(P \| Q) \geq  \max_{a \in \mathcal{A} } \Df(P_a \| Q_a^\star) - \delta. \]
\end{theorem}

\begin{proof}
    Let $P \in \PX$ be a target probability distribution satisfying equalized generative odds and $f$ be a function such that $\mathcal{D}_f$ defines an $f$-divergence. By definition of the attribute mapping $\psi$, $\{ \mathcal{X}_a \mid a \in \mathcal{A} \}$ is a partition of $\mathcal{X}$. Hence we can rewrite $P = \sum_{a \in \mathcal{A}} \pi_a^P P_a$ with $\Supp(P_a) \subset \mathcal{X}_a$. Let us now consider $Q \in \mathcal{Q}$.  As all probability distributions in $\mathcal{Q}$ satisfy matching odds with respect to $P$, we use Lemma~\ref{lemma:decmpositiondf} to write
    \begin{align}
        \Df(P \Vert Q) = \sum_{a \in \mathcal{A}} \pi_a^{Q} \Df(P_a,Q_a).
        \label{eq:decompositiondivergence}
    \end{align}

    \emph{Reasoning ab absurdum.} Based on the above, we reason ab absurdum to obtain the expected result. Let us take $a_{\#} \in \argmax\limits_{a \in \mathcal{A}} \Df(P_a \Vert Q_a^\star)$, where $Q^\star \in  \argmin\limits_{ Q \in \overline{\mathcal{Q}}_\mathcal{A}} \Df\left( P \| Q \right)$, and assume that \begin{align}
        \label{eq:HypothesisAbsurdum}
        \Df(P \Vert Q) < \Df\left( P_{a_{\#}} \Vert Q^\star_{a_{\#}} \right)  - \delta.
    \end{align}
    Note that, by definition of $\delta$-equalized generative treatment, we know that
    \[ \Df\left( P_{a_{\#}} \Vert Q_{a_{\#}} \right) - \Df\left( P_{a} \Vert Q_{a} \right) \leq \delta,~ \forall a \in \mathcal{A}. \]

    Hence,  we also have $\Df\left( P_{a_{\#}} \Vert Q_{a_{\#}} \right) \leq \Df\left( P_{a} \Vert Q_{a} \right) + \delta,~ \forall a \in \mathcal{A}$. Multiplying both sides by $\pi_a^{Q}$ for all $a \in \mathcal{A}$ and summing the terms one gets

    \[  \sum_{a \in \mathcal{A}} \pi_a^{Q} \Df\left( P_{a_{\#}} \Vert Q_{a_{\#}} \right) \leq \sum_{a \in \mathcal{A}} \pi_a^{Q} \left( \Df\left( P_{a} \Vert Q_{a} \right) + \delta \right). \]
    Finally, using the fact that $\sum_{a \in \mathcal{A}} \pi_a^{Q} = 1$ and the decomposition in~(\ref{eq:decompositiondivergence}), we obtain
    \begin{align}
        \label{eq:absurdum2}
        \Df\left( P_{a_{\#}} \Vert Q_{a_{\#}} \right) \leq \Df\left( P \Vert Q \right) + \delta.
    \end{align}
    Now, let us consider the distribution $\tilde{Q} \in \PX$ defined as \[\tilde{Q} = \sum_{\substack{a \in \mathcal{A} \\ a \neq a_{\#}}} \pi_a^{Q^\star} Q_a^\star + \pi_{a_{\#}}^{Q^\star} Q_{a_{\#}}.\] By definition of $\overline{\mathcal{Q}}^{\mathcal{A}}$, we have that $\tilde{Q} \in \overline{\mathcal{Q}}^{\mathcal{A}}$ and that $Q^\star$ satisfies matching generative odds with respect to $P$.
    Hence, $\tilde{Q}$ also satisfies matching generative odds with respect to $P$, which gives us (using Lemma~\ref{lemma:decmpositiondf})
    \[ \Df\left( P \Vert \tilde{Q} \right) = \sum_{ \substack{a \in \mathcal{A} \\ a \neq a_{\#}}} \pi_a^{Q^\star} \Df(P_a \Vert Q_a^\star) + \pi_{a_{\#}}^{Q^\star} \Df(P_{a_{\#}}\Vert Q_{a_{\#}}).\]
    Substituting successively~(\ref{eq:absurdum2}) and~(\ref{eq:HypothesisAbsurdum}) in the above, and using the fact that $Q^\star$ satisfies  matching generative odds with respect to $P$ (to obtain the last equality), we have
    \begin{align*}
        \Df\left( P \Vert \tilde{Q} \right) & \leq \sum_{ \substack{a \in \mathcal{A}                                                                   \\ a \neq a_{\#}}} \pi_a^{Q^\star} \Df(P_a \Vert Q_a^\star) + \pi_{a_{\#}}^{Q^\star} \left( \Df(P \Vert Q) + \delta \right) \\
                                            & < \sum_{ \substack{a \in \mathcal{A}                                                                      \\ a \neq a_{\#}}} \pi_a^{Q^\star} \Df(P_a \Vert Q_a^\star) + \pi_{a_{\#}}^{Q^\star} \Df\left( P_{a_{\#}} \Vert Q^\star_{a_{\#}} \right) \\
                                            & < \sum_{a \in \mathcal{A}} \pi_a^{Q^\star} \Df(P_a \Vert Q_a^\star) =  \Df\left( P \Vert Q^\star \right).
    \end{align*}
    According to the above, we just constructed a probability distribution $\tilde{Q} \in \overline{\mathcal{Q}}^{\mathcal{A}}$ that has an $f$-divergence strictly smaller than $Q^\star$ with respect to $P$. By definition of $Q^\star$, this is impossible, hence contradicting our initial assumption that $\Df(P \Vert Q) < \Df\left( P_{a_{\#}} \Vert Q^\star_{a_{\#}} \right)  - \delta.$ In particular, this means that \[ \Df(P \| Q) \geq  \max_{a \in \mathcal{A} } \Df(P_a \| Q_a^\star) - \delta. \]
\end{proof}

%% file: content/appendix/experiment.tex
\section{Experimental Setup}
\label{app:sec:experiments}

\subsection{Building Illustrative Examples}
\label{app:sec:experiments:cata}

\begin{figure}[b!]
    \begin{center}
        \subfloat[Example 1]{\includegraphics[width=0.47\textwidth]{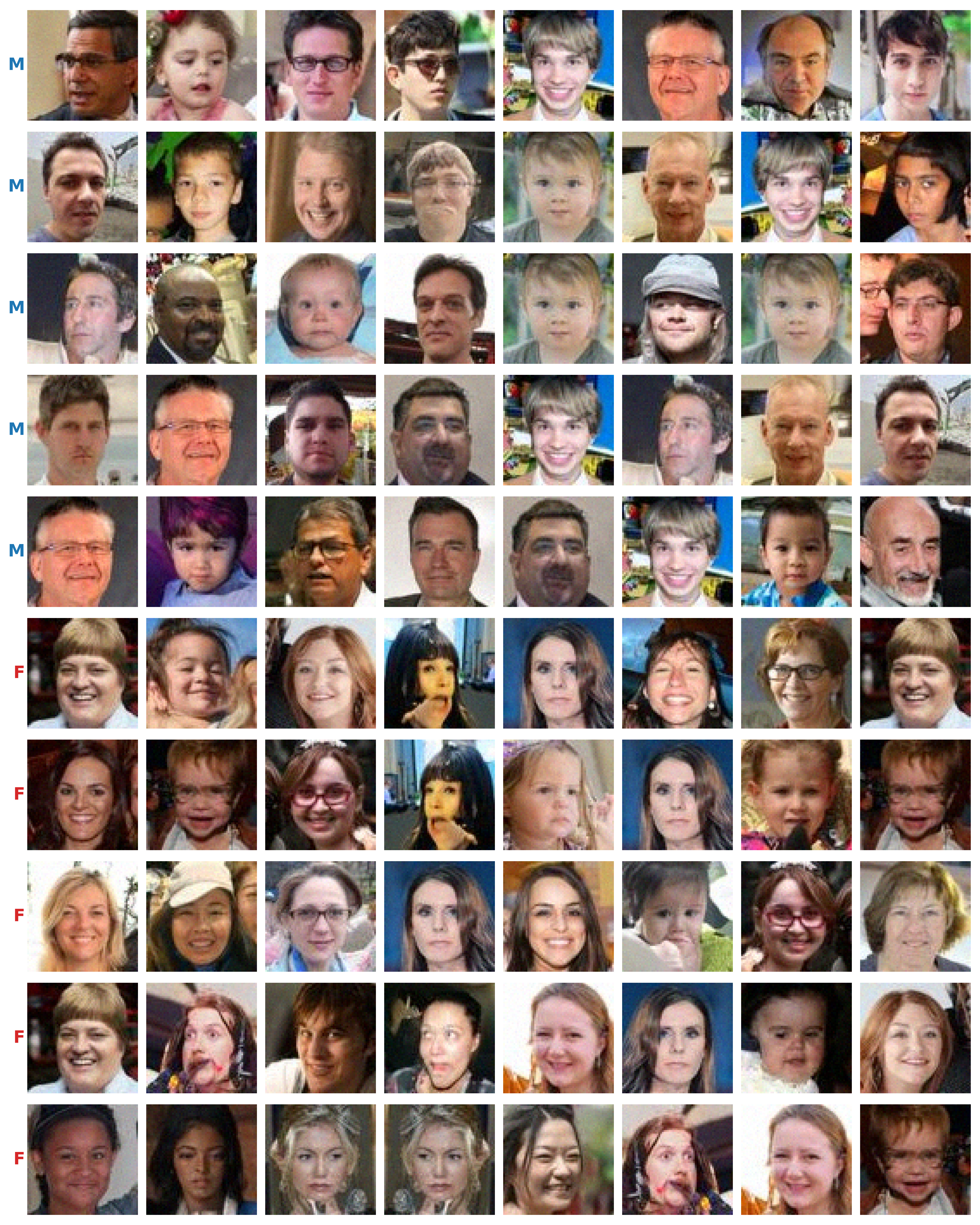}\label{app:fig:cata:all}}\hspace{0.05\textwidth}%
        \subfloat[Example 2]{\includegraphics[width=0.47\textwidth]{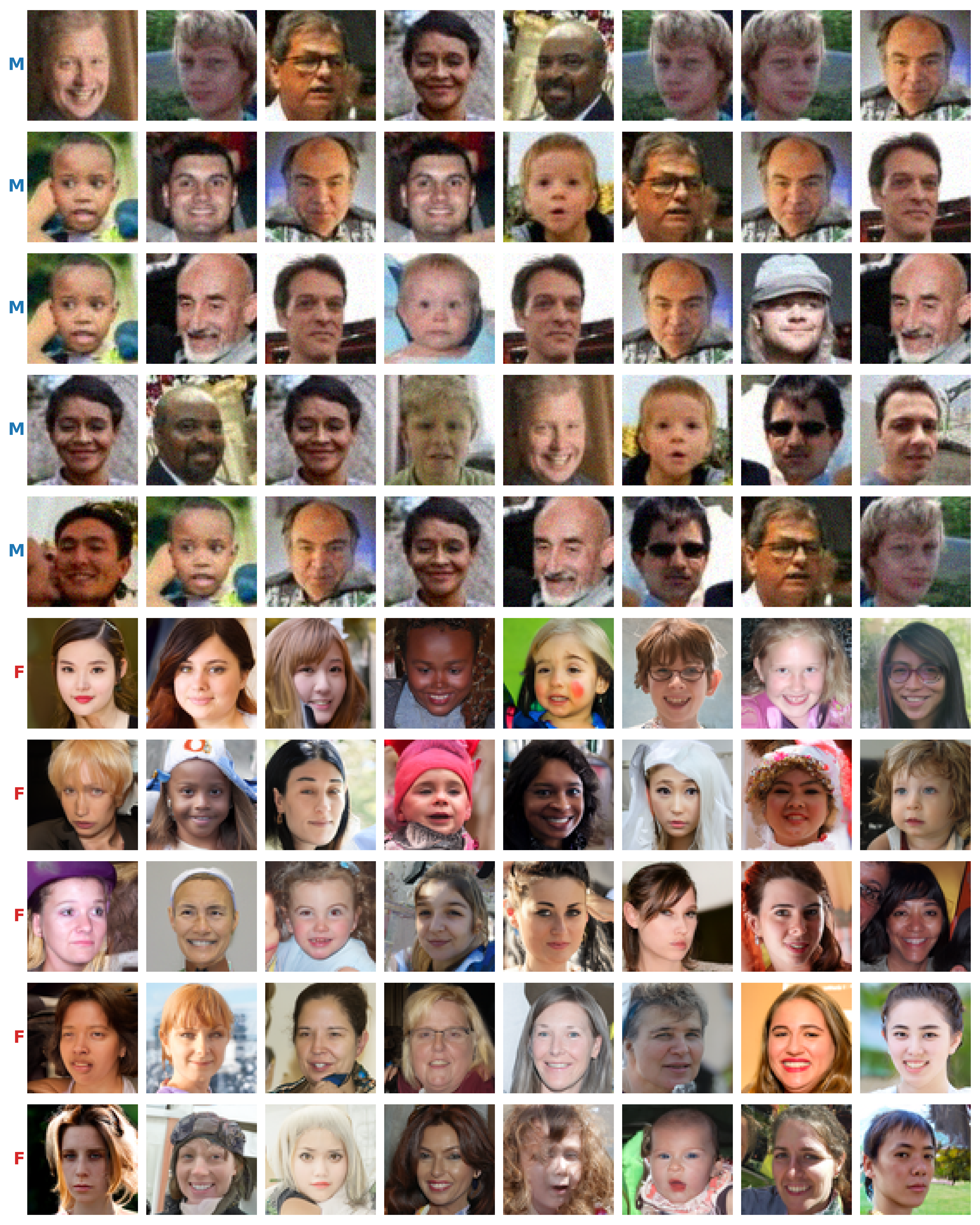}\label{app:fig:cata:male}}%
        \caption{Samples from distribution of FFHQ faces with similar precision ($79.06$ (\ref{app:fig:cata:all}) and  $77.12$ (\ref{app:fig:cata:male})) and similar recall $R$ ($60.02$ (\ref{app:fig:cata:all}) and $60.31$ (\ref{app:fig:cata:male})). However, the precision and recall for the sub-group are very different. In particular, in Example 1 the models generates slighly noised images for both classes. In Example 2, the model generates limited diversity and noisy images for Male Class while it generates high quality images and diverse samples for Female.}
        \label{app:fig:cata}
    \end{center}
\end{figure}
In this section, we provide an illustrative examples. Two synthetic scenarios are constructed to highlight limitations of proportion-based fairness constraints when subgroup generation quality differs. Both examples are based on FFHQ~\citep{karras_style-based_2019}, which provides face images with gender metadata. Samples are generated at resolution $64\times 64$ using a pretrained unconditional EDM-VP model~\citep{karras_elucidating_2022}. To enforce MGO, the generated set is first filtered so that male/female proportions match the desired ratio. Disparities are then introduced \emph{post hoc} in a controlled way.

\textbf{Quality gap (precision).} Random Gaussian noise is added to one group (or both) to create a difference in fidelity. In Example~1, mild noise is applied to both groups, leading to a moderate precision drop for both. In Example~2, strong noise is applied to male images while female images remain unchanged, producing a large precision gap.

\textbf{Diversity gap (recall).} To reduce diversity, a subset of generated images is duplicated and randomly horizontally flipped. In Example~1, $20\%$ of images are duplicated in both groups, yielding a mild recall drop for both. In Example~2, only male images are duplicated, creating a large recall gap.

Final global and subgroup precision/recall values are reported in Table~\ref{app:tab:cata}. Although the two examples have similar global P\&R, subgroup metrics differ substantially, illustrating that matching proportions (MGO/EGO) can hide large disparities in conditional generation quality.

\begin{table}[t]
    \centering
    \caption{precision and recall for the two examples shown in Figure~\ref{app:fig:cata}. The global precision and recall are similar, but the subgroup metrics differ significantly.}
    \label{app:tab:cata}
    \begin{tabular}{lcccccc}
        \toprule
        Subset    & Precision & Recall & M Recall & F Recall & M Precision & F Precision \\
        \midrule
        Example 1 & 79.06     & 60.02  & 47.96    & 45.01    & 73.96       & 71.88       \\
        Example 2 & 77.12     & 60.31  & 1.52     & 91.22    & 66.53       & 95.85       \\
        \bottomrule
    \end{tabular}
\end{table}

\subsection{Related works and our method}

In this work, several baselines are considered to contrast proportion matching (MGO/EGO) with our treatment-based criterion (EGT) . We briefly review them here and in Table~\ref{tab:methods} summarize their properties.

\textbf{Standard unconditional} training optimizes the model on the full dataset mixture. As a result, the generated attribute proportions typically mirror the dataset proportions in expectation, but there is no mechanism to reliably \emph{enforce} target proportions at inference time, nor to prevent conditional quality gaps.

\textbf{Conditional} models explicitly learn attribute-conditional distributions. By sampling attributes in prescribed ratios at generation time, both MGO and EGO can be enforced by construction.

\textbf{Rejection sampling (RS)} enforces target proportions at inference time by filtering generated samples until the desired ratios are reached \citep{zameshina_fairness_2022}. This requires an oracle classifier (specified in the diffusion/LLM sections) and can be computationally expensive. For Wikipedia Biographies, pretrained LLMs substantially under-generate female biographies (down to about $16\%$), so achieving EGO may require generating more than $5\times$ as many samples. For this reason, RS is only applied to diffusion models in our experiments.

\begin{table}[!b]
    \centering
    \caption{Methods for generative fairness. The table indicates whether the method applies at training or inference time, and whether it can enforce Matching Generative Odds (MGO) and/or Equalized Generative Odds (EGO), and Equalized Generative Treatment (EGT).}
    \label{tab:methods}
    \begin{tabular}{l|c|cccc}
        Method                                             & Applies on & MGO    & EGO    & EGT    & \\\hline
        Standard Unconditional                             & Training   & \xmark & \xmark & \xmark   \\
        Rejection Sampling~\citep{zameshina_fairness_2022} & Inference  & \cmark & \cmark & \xmark   \\
        Standard Conditional                               & Training   & \cmark & \cmark & \xmark   \\
        Reweighted Loss~\citep{choi_fair_2020}             & Training   & \xmark & \cmark & \xmark   \\
        Min-Max                                     & Training   & \xmark & \xmark & \cmark  \\
        Min-Max + Rejection Sampling                                     & Training   & \cmark & \cmark & \cmark  \\
    \end{tabular}
    \centering
\end{table}

\textbf{Reweighted Loss} \citep{choi_fair_2020} rescales per-sample losses with group-dependent weights. Concretely, losses from group $a$ are multiplied by a factor proportional to $\pi_a^P/\pi_a^Q$, where $\pi^Q$ denotes the training mixture proportions and $\pi^P$ is the target mixture (e.g., uniform for balancing). This encourages the optimizer to allocate more capacity to underrepresented groups and often improves proportion-based criteria, but it does not provide guarantees for EGT. For Diffusion models, we use a rescaling per noise level, similarly to the work of \citet{kim_training_2024}.

\subsection{Diffusion Models}
\label{app:sec:experiments:diff}

\subsubsection{Experimental Set-Up}
We use the publicly available codebase of \citet{karras_elucidating_2022} to train and evaluate diffusion models. We experiment with two different architectures: EDM-VP and EDM-VE. Both architectures are based on a U-Net architecture with attention layers, but they differ in the noise schedule and the parameterization of the denoising model. We refer the reader to \citet{karras_elucidating_2022} for more details on the architectures. We use the pre-trained models provided by \citet{karras_elucidating_2022} as our baseline models for EDM-VP and EDM-VE on FFHQ in our experiments.

\subsubsection{Training Diffusion Models}
\label{app:sec:experiments:diff:train}

Training diffusion models consists in defining a sequence of noise levels $\sigma$ and minimizing a weighted sum of denoising objectives. Concretely, we introduce a denoising model $F_\theta(x,\sigma)$ and optimize
\begin{equation}
    \mathcal{L}_{\mathrm{Unconditional}}(\theta) = \E_{x\sim P, \sigma}\left[ w(\sigma)\, \|F_\theta(x,\sigma) - x\|^2\right],
\end{equation}
where $w(\sigma)$ are predefined weights. This objective encourages the model to predict the clean signal from its noisy counterpart at various noise levels. In conditional training, the same principle applies, but the denoising function is extended to incorporate the group attribute using the oracle function:
\begin{equation}
    \mathcal{L}_{\mathrm{Conditional}}(\theta) = \E_{x\sim P, \sigma}\left[  w(\sigma) \, \|F_\theta(x,\psi(x),\sigma) - x\|^2\right].
\end{equation}
This setup allows explicit conditioning on sensitive attributes. In practice, both unconditional and conditional models have a similar number of parameters, with conditional models requiring slightly more due to the embedding of attribute information.

To adapt Min--Max training to diffusion models, the maximization step must account for the multi-noise nature of the objective. Indeed, a large loss at a given noise level can translate into different failure modes in the final samples, potentially affecting either fidelity (precision) or coverage (recall). For this reason, the worst-group selection is performed \emph{separately at each noise level}, and the model is trained to reduce the loss of the underperforming group across the entire diffusion trajectory. This procedure is summarized in Algorithm~\ref{alg:minmax-diffusion-noise}.

\begin{algorithm}[H]
    \caption{Min--max diffusion training across noise levels}
    \label{alg:minmax-diffusion-noise}
    \begin{algorithmic}[1]
        \STATE \textbf{Initialize} model parameters $\theta$ of $F_{\theta}$
        \STATE \textbf{Initialize} moving averages $\{L_{a,\sigma}\}_{a\in\mathcal{A},\,\sigma\in\Sigma}$ (e.g., $L_{a,\sigma}\gets 0$)
        \FOR{each training iteration}
        \STATE Sample a minibatch $\mathcal{B}$ of data points $x$ with group labels $a\in\mathcal{A}$ and noise levels $\sigma\in\Sigma$
        \STATE Compute per-group, per-noise losses $\ell_{a,\sigma} = \|F_{\theta}(x,a,\sigma) - x\|_2^2$ on $\mathcal{B}$
        \FORALL{$(a,\sigma) \in \mathcal{A}\times\Sigma$}
        \STATE $L_{a,\sigma} \leftarrow \alpha L_{a,\sigma} + (1-\alpha)\,\ell_{a,\sigma}$
        \ENDFOR
        \FORALL{$\sigma \in \Sigma$}
        \STATE $a^{\star}(\sigma) \leftarrow \arg\max_{a \in \mathcal{A}} L_{a,\sigma}$
        \ENDFOR
        \STATE $\mathcal{L} \leftarrow \sum_{\sigma \in \Sigma} \ell_{a^{\star}(\sigma),\,\sigma}$
        \STATE Compute gradient $g \leftarrow \nabla_{\theta}\,\mathcal{L}$
        \STATE Update parameters $\theta \leftarrow \textsc{OptimizerStep}(\theta, g)$
        \ENDFOR
    \end{algorithmic}
\end{algorithm}
 
As training relies on mini-batch SGD, the number of samples from an underperforming sensitive group within a batch can be small and highly variable. This issue is further amplified when losses are tracked per noise level, which increases variance in the maximizer selection. To stabilize training and avoid noisy oscillations in the worst-group choice, an exponential moving average (EMA) is used to smooth group-wise losses over iterations. Algorithm~\ref{alg:minmax-ema} describes this EMA-based Min--Max strategy, which is applied both to diffusion training and to LLM fine-tuning.

\begin{algorithm}[H]
    \caption{Min--max training with EMA of group-wise losses}
    \label{alg:minmax-ema}
    \begin{algorithmic}[1]
        \STATE \textbf{Initialize} model parameters $\theta$ of $Q_{\theta}$
        \STATE \textbf{Initialize} moving averages $\{L_a\}_{a\in\mathcal{A}}$ (e.g., $L_a \gets 0$)
        \FOR{each training iteration}
        \STATE Sample a minibatch $\mathcal{B}$ with group labels $a \in \mathcal{A}$
        \STATE Compute per-group losses $\{\ell_a\}_{a\in\mathcal{A}}$ on $\mathcal{B}$
        \FORALL{$a \in \mathcal{A}$}
        \STATE $L_a \leftarrow \alpha L_a + (1-\alpha)\,\ell_a$
        \ENDFOR
        \STATE $a^{\star} \leftarrow \arg\max_{a \in \mathcal{A}} L_a$
        \STATE Compute gradient $g \leftarrow \nabla_{\theta}\,\ell_{a^{\star}}$
        \STATE Update parameters $\theta \leftarrow \textsc{OptimizerStep}(\theta, g)$
        \ENDFOR
    \end{algorithmic}
\end{algorithm}

\paragraph{Training details and loss behavior.} All diffusion models are trained for 1M iterations with batch size 256 using Adam at learning rate $10^{-4}$. For Min--Max training, an EMA with decay $\alpha=0.9$ tracks group-wise losses (and, for diffusion, per-noise losses) to stabilize the maximization step. Training is performed on 8 NVIDIA V100 GPUs (32GB) and takes roughly 5 hours per model. Table~\ref{tab:train_loss_uncond_cond} reports normalized training losses averaged over noise levels and groups (lowest value set to 1). Conditional training typically flattens the loss landscape across groups and noise levels, while Min--Max and reweighting also reduce loss disparities, though the effect size varies across VP/VE.

\begin{table}[H]
    \footnotesize
    \setlength{\tabcolsep}{3pt}
    \caption{Training loss (MSE reconstruction error) for different diffusion models on FFHQ dataset (conditional and unconditional). The values are averaged over all noise levels and classes. The loss is normalized with respect to the class with the lowest loss.}
    \centering
    \input{tables/ffhq_merged_all_methods.tex}
    \label{tab:train_loss_uncond_cond}
\end{table}

In Figure~\ref{app:fig:losses_ffhq_ve}, we plot the estimated loss for the different methods across noise levels for EDM-VE models. We observe that baseline models already have relatively balanced losses (compared to EDM-VP models in Figure~\ref{fig:ffhq_losses_comparison} in the main text).

\begin{figure}[h]
    \centering
    \includegraphics[width=0.95\textwidth]{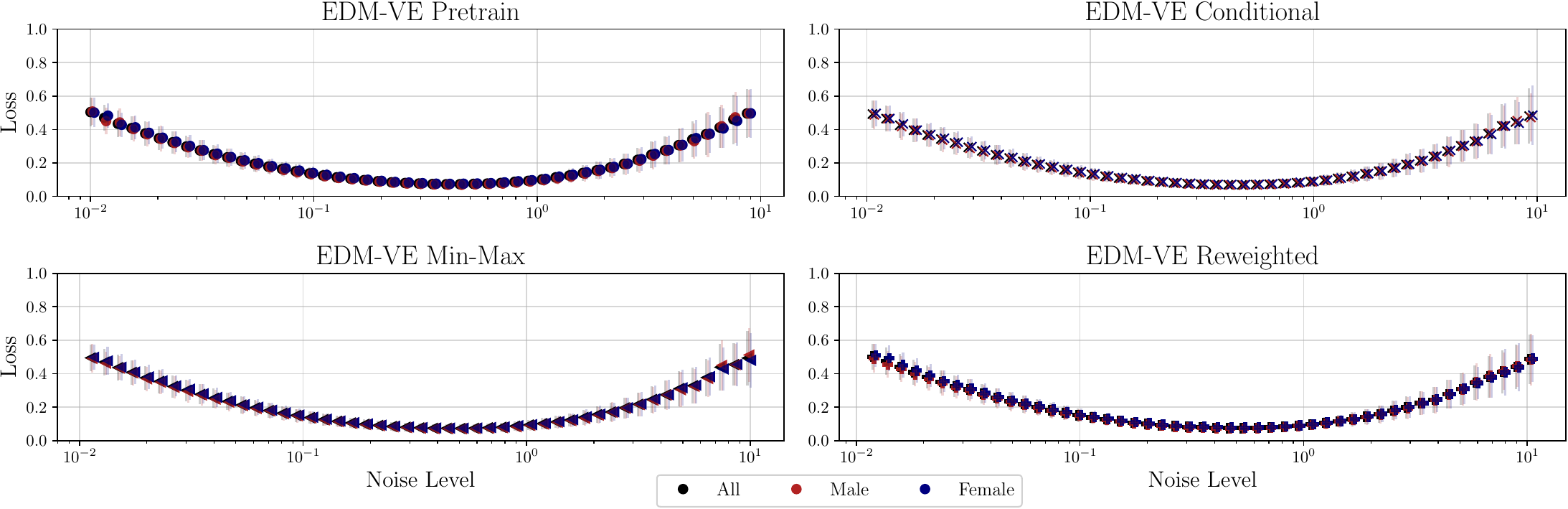}
    \caption{Comparison of the estimated loss for the baseline, Min-Max, conditional and reweighted methods on the FFHQ dataset with the VE model. The loss is plotted for every noise level.}
    \label{app:fig:losses_ffhq_ve}
\end{figure}

\subsubsection{Evaluating Diffusion Models}
\label{app:sec:experiments:diff:eval}
\label{app:sec:experiments:diffusion:eval}

We evaluate fidelity and diversity using the Precision/Recall (P\&R) framework of \citet{kynkaanniemi_improved_2019}, together with the more robust Topological Precision \& Recall (TopP\&R) procedure of \citet{kim_toppr_2023}. Precision measures the fraction of generated features lying in the estimated support of real features (fidelity), while recall measures the fraction of real features covered by the support of generated features (diversity/coverage). TopP\&R improves support estimation by retaining only statistically and topologically significant structures, making P\&R estimates more stable.

To detect subtle distinctions in FFHQ, features are extracted with a DINOv2 ViT \citep{oquab_dinov2_2024} fine-tuned on FFHQ. The same fine-tuned DINOv2 is used as an oracle to assign sensitive-attribute labels (male/female) to generated images. For computational stability, features are projected using random Gaussian projections before running TopP\&R. The whole procedure is repeated \textbf{10} times with independent projections; we report means (and standard deviations when relevant). Unless stated otherwise, evaluation uses \textbf{20k generated samples per class} with a matched number of real samples per class. FID is also reported for completeness, but fairness conclusions rely on subgroup TopP\&R precision/recall and their sum.

\paragraph{Results.} We present here the full tables of results for all diffusion models trained on FFHQ dataset. Table~\ref{app:tab:metrics_edmvp} report the metrics for every models. The results are discussed in Section~\ref{sec:improved} in the main text.
\begin{table}[H]
    \small
    \centering
    \setlength{\tabcolsep}{3pt}
    \caption{Evaluation metrics for different Unconditional EDM-VP models on FFHQ dataset.}
    \label{app:tab:metrics_edmvp}
    \input{tables/ffhq_vpve_appendix.tex}
\end{table}

\subsection{Large Language Models}
\label{app:sec:experiments:llm}

\subsubsection{Training Large Language Models}
\label{app:sec:experiments:llm:train}

We fine-tune the chat variants of LLaMA~3.2 at two scales (\textbf{1B} and \textbf{3B} parameters) using parameter-efficient adapters (LoRA). Intuitively, LoRA inserts small trainable matrices into the attention/MLP layers so that we only update a tiny fraction of weights, which is more memory- and time-efficient than full fine-tuning.

\paragraph{Prompting and conditioning.} Fine-tuning uses a two-message chat template with a system instruction specifying a Wikipedia-style tone and a user query requesting a short biography summary. The model is trained to produce the response:
\begin{verbatim}
SYSTEM_PROMPT = """
You are a helpful assistant that writes Wikipedia-style biography 
summaries of notable individuals. 
Your writing should follow the tone, style, and structure of 
Wikipedia "Summary" sections:
- Neutral, encyclopedic tone
- No personal opinions or promotional language
- Concise but informative, focusing on key life events, career, 
and achievements
- Chronological ordering of major facts
- Use proper sentences, not bullet points
- The biography text should start with "Biography:" followed 
directly by the text in the next line.
Do not invent references or citations. Do not include external 
commentary about the writing task.
"""

CHAT_ZERO_SHOT = (
    'Write the "Summary" section of a Wikipedia article about 
    a person of your choice.'
)

messages = [
    {"role": "system", "content": SYSTEM_PROMPT},
    {"role": "user",   "content": CHAT_ZERO_SHOT},
]
\end{verbatim}

To obtain \emph{conditional} generations, the same template is used while appending a short attribute constraint (here, gender) directly in the user message:
\begin{verbatim}
CHAT_ZERO_SHOT = (
    'Write the "Summary" section of a Wikipedia article about 
    a person of your choice.'
)
if gender is not None:
    user_prompt = CHAT_ZERO_SHOT + f" It should be a biography 
    of a {gender} person."
messages = [
    {"role": "system", "content": SYSTEM_PROMPT},
    {"role": "user",   "content": user_prompt},
]
\end{verbatim}

This prompt-based conditioning makes it straightforward to control attribute proportions at inference time (e.g., 50/50 or 75/25), and directly matches the EGO/MGO controls discussed earlier.

\paragraph{Training details.} We fine-tune the models for \textbf{6 epochs} with a batch size of \textbf{64} and a learning rate of \textbf{1e-4} using the AdamW optimizer. We use a LoRA rank of \textbf{8}. Training is performed on a single NVIDIA H100 GPU with 80GB of memory and takes approximately \textbf{2 hours} for the 1B model and \textbf{5 hours} for the 3B model.

\subsubsection{Evaluating Large Language Models}
\label{app:sec:experiments:llm:eval}

We evaluate text generations with the same P\&R formalism and TopP\&R support estimation used for images, but in a \emph{text-embedding space}. Concretely, we embed both real (Wikipedia) and generated biographies with a strong sentence-embedding model (we use a 4B-parameter open-source text-embedding model; reference omitted for anonymity). We then apply a random Gaussian projection to reduce dimensionality and run TopP\&R to estimate topological supports and compute precision/recall. This process is repeated \textbf{10} times with independent projections; we report the mean (and, when relevant, the standard deviation).

\paragraph{Oracle for sensitive attribute.} For LLM outputs, we assign gender labels using a simple pronoun-based heuristic: we scan each biography and count occurrences of gendered terms (e.g., \emph{he/him/his} vs. \emph{she/her/hers}). The majority class is used as the predicted label; ties or ambiguous cases are left unassigned. While noisy, this heuristic works well on short Wikipedia-style summaries and avoids extra training.

\paragraph{Notes.} As with images, we compute group-wise P\&R (and P+R) under TopP\&R on embeddings to quantify fidelity and coverage per sensitive group. We also report FID-like global metrics for completeness, but our fairness conclusions rely on group-wise TopP\&R.

\begin{table}[H]
    \centering
    \small
    \caption{Evaluation metrics for different LLMs on Wikipedia Bio dataset for the different methods. We plot the results for LLaMA-1B and LLaMA-3B models, as well as for Gemma-4B model pretrained and instruct-tuned.}
    \input{tables/llms_app.tex}
    \label{tab:train_eval_llm}
\end{table}
\label{app:tab:train_loss_llm}

\subsubsection{Wikipedia Biographies Dataset}
\label{app:sec:experiments:llm:wiki}

The Wikipedia Biographies dataset is derived from the Wikipedia Biographies dataset introduced by \cite{bronnec_exploring_2024}, which contains header of biographies of individuals. Following, we present examples of biographies generated by the different models.

\paragraph{Unconditional Generation:}
\input{tables/examples_by_model_tabularx.tex}

\paragraph{Conditional Generation:}
\input{tables/examples_by_model_tabularx_cond.tex}

%% file: tables/ffhq_merged_all_methods.tex
\begin{tabular}{l l ccc}
\toprule
Model & Method & All & Male & Female \\
\midrule
\multirow{4}{*}{EDM-VP} & Pretrain & 2.29$\pm$0.80 & 3.88$\pm$1.80 & 1.00$\pm$0.00 \\
 & Conditional & 1.02$\pm$0.02 & 1.00$\pm$0.00 & 1.03$\pm$0.03 \\
 & Reweighted & 1.03$\pm$0.03 & 1.00$\pm$0.00 & 1.05$\pm$0.05 \\
 & Min-Max & 1.04$\pm$0.04 & 1.00$\pm$0.00 & 1.07$\pm$0.08 \\
\midrule
\multirow{4}{*}{EDM-VE} & Pretrain & 1.02$\pm$0.02 & 1.00$\pm$0.00 & 1.04$\pm$0.03 \\
 & Conditional & 1.02$\pm$0.01 & 1.00$\pm$0.00 & 1.03$\pm$0.02 \\
 & Reweighted & 1.02$\pm$0.02 & 1.00$\pm$0.00 & 1.04$\pm$0.03 \\
 & Min-Max & 1.02$\pm$0.01 & 1.00$\pm$0.00 & 1.03$\pm$0.02 \\
\bottomrule
\end{tabular}

%% file: tables/ffhq_vpve_appendix.tex
\begin{tabular}{l l l |c c c c| c c c c |c c c c |c}
\toprule
& & & \multicolumn{4}{c}{Precision} & \multicolumn{4}{c}{Recall} & \multicolumn{4}{c}{FID} & $\delta$-PR \\
Model & Gen & Method & All & Male & Female & $\delta$ & All & Male & Female & $\delta$ & All & Male & Female & $\delta$  &   \\
\midrule
VP & MGO & Pretrained & 95.2 & 95.2 & 97.2 & 2.0 & 84.3 & 91.7 & 87.5 & 4.2 & \underline{2.33} & \underline{3.12} & \underline{2.78} & 0.34 & 2.2 \\
VP & MGO & Conditional & \textbf{98.0} & \textbf{96.8} & \textbf{98.4} & \textbf{1.6} & \underline{80.2} & 91.5 & \underline{84.7} & 6.8 & 2.75 & 3.41 & 3.32 & \underline{\textbf{0.08}} & 5.2 \\
VP & MGO & Reweighted & 94.5 & \underline{94.9} & \underline{96.7} & \textbf{1.8} & \textbf{87.8} & \textbf{92.2} & \textbf{89.8} & \textbf{2.3} & 2.44 & 3.28 & 2.84 & 0.44 & \textbf{0.6} \\
VP & MGO & Min-Max & \underline{89.4} & \textbf{96.8} & 96.7 & \underline{\textbf{0.1}} & \textbf{91.2} & \underline{86.6} & 86.4 & \underline{\textbf{0.2}} & 2.75 & 3.51 & 3.27 & \textbf{0.24} & \underline{\textbf{0.3}} \\
\midrule
VP & EGO & Pretrained & 96.1 & \underline{95.6} & 97.8 & 2.2 & \underline{83.6} & 91.7 & 87.2 & 4.5 & \underline{2.50} & \underline{2.99} & \underline{2.78} & 0.21 & 2.3 \\
VP & EGO & Conditional & 95.9 & \textbf{96.4} & \textbf{98.2} & \textbf{1.7} & \textbf{83.6} & \textbf{92.3} & \underline{84.6} & 7.7 & 2.85 & 3.36 & 3.37 & \underline{\textbf{0.00}} & 6.0 \\
VP & EGO & Reweighted & 93.6 & \textbf{96.1} & \underline{96.4} & \textbf{0.3} & \textbf{86.9} & 90.8 & \textbf{89.9} & \textbf{1.0} & 2.62 & 3.25 & 2.89 & 0.36 & \textbf{0.6} \\
VP & EGO & Min-Max & \underline{92.7} & \textbf{96.7} & 97.0 & \underline{\textbf{0.3}} & \textbf{88.3} & \underline{86.7} & 85.8 & \underline{\textbf{0.8}} & 2.96 & 3.35 & 3.27 & \textbf{0.08} & \underline{\textbf{0.5}} \\
\midrule
VE & MGO & Pretrained & 97.9 & 95.9 & 98.0 & 2.1 & \underline{81.4} & 87.8 & \underline{83.4} & 4.4 & \underline{2.49} & 3.27 & \underline{2.92} & 0.35 & 2.3 \\
VE & MGO & Conditional & 95.6 & \textbf{98.2} & 97.4 & \textbf{0.9} & \textbf{86.4} & \underline{87.1} & \textbf{84.4} & \textbf{2.7} & 2.88 & 3.54 & 3.43 & \textbf{0.11} & 3.5 \\
VE & MGO & Reweighted & 90.0 & \underline{94.6} & \underline{95.3} & \textbf{0.8} & \textbf{90.5} & \textbf{90.2} & \textbf{87.2} & \textbf{3.0} & 2.55 & \underline{\textbf{3.22}} & 3.17 & \underline{\textbf{0.04}} & \textbf{2.2} \\
VE & MGO & Min-Max & \underline{89.2} & \textbf{96.9} & 97.2 & \underline{\textbf{0.3}} & \textbf{91.4} & \textbf{89.5} & \textbf{88.4} & \underline{\textbf{1.1}} & 2.81 & 3.67 & 3.34 & \textbf{0.33} & \underline{\textbf{0.8}} \\
\midrule
VE & EGO & Pretrained & 96.4 & 97.2 & 97.6 & 0.4 & 81.6 & \underline{85.8} & 83.6 & 2.2 & \underline{2.65} & \underline{3.15} & \underline{2.92} & 0.23 & 1.8 \\
VE & EGO & Conditional & \textbf{97.5} & 95.5 & \textbf{98.7} & 3.2 & \underline{80.7} & \textbf{90.9} & \underline{81.2} & 9.7 & 3.00 & 3.50 & 3.47 & \underline{\textbf{0.03}} & 6.5 \\
VE & EGO & Reweighted & 93.2 & \underline{94.7} & \underline{94.7} & \underline{\textbf{0.0}} & \textbf{88.7} & \textbf{89.8} & \textbf{90.4} & \underline{\textbf{0.5}} & 2.77 & 3.20 & 3.25 & \textbf{0.05} & \textbf{0.5} \\
VE & EGO & Min-Max & \underline{87.2} & \textbf{98.1} & 96.1 & 2.0 & \textbf{91.4} & \textbf{87.8} & \textbf{90.1} & 2.3 & 3.08 & 3.53 & 3.34 & \textbf{0.19} & \underline{\textbf{0.3}} \\
\bottomrule
\end{tabular}

%% file: tables/llms_app.tex
\begin{tabular}{l l cccc cccc c}
\toprule
& & \multicolumn{4}{c}{Precision} & \multicolumn{4}{c}{Recall} & $\delta$PR \\
Model & Method & All & Male & Female & $\delta$ & All & Male & Female & $\delta$ &  \\
\midrule
\multirow{4}{*}{LLaMA-3.2-Chat 1B} & Pretrained & 50.45 & \underline{55.68} & \underline{76.86} & 21.18 & 86.08 & 77.03 & 65.57 & 11.46 & 32.65 \\
 & Conditional & 45.85 & 55.68 & 74.76 & \underline{\textbf{19.09}} & \underline{\textbf{86.29}} & 74.18 & \textbf{66.36} & \textbf{7.81} & \underline{\textbf{26.90}} \\
 & Reweighted & \textbf{50.51} & 55.43 & 76.72 & 21.29 & 85.29 & 76.12 & \textbf{66.91} & \textbf{9.22} & \textbf{30.50} \\
 & MinMax & \underline{\textbf{50.96}} & 54.79 & 74.34 & \textbf{19.55} & 84.60 & \underline{\textbf{77.11}} & \underline{\textbf{69.39}} & \underline{\textbf{7.73}} & \textbf{27.28} \\
\midrule
\multirow{4}{*}{LLaMA-3.2-Chat 3B} & Pretrained & \underline{51.77} & 53.06 & \underline{78.43} & 25.37 & 91.37 & \underline{87.40} & 72.34 & 15.06 & 40.43 \\
 & Conditional & 46.48 & \textbf{53.14} & 77.44 & \textbf{24.31} & \underline{\textbf{93.67}} & 85.38 & \underline{\textbf{73.11}} & \underline{\textbf{12.26}} & \textbf{36.57} \\
 & Reweighted & 50.16 & \textbf{54.75} & 77.63 & \textbf{22.88} & \textbf{91.88} & 84.94 & \textbf{72.65} & \textbf{12.29} & \textbf{35.16} \\
 & MinMax & 49.45 & \underline{\textbf{54.78}} & 77.05 & \underline{\textbf{22.27}} & \textbf{91.86} & 85.76 & \textbf{72.89} & \textbf{12.87} & \underline{\textbf{35.13}} \\
\midrule
\multirow{4}{*}{Gemma-3 4B (pt)} & Pretrained & 63.98 & 66.22 & 87.12 & 20.90 & 96.71 & 97.03 & 84.90 & 12.13 & 33.03 \\
 & Conditional & 57.50 & 65.46 & \underline{\textbf{88.64}} & 23.17 & \underline{\textbf{98.71}} & \underline{\textbf{97.39}} & \underline{\textbf{86.30}} & \underline{\textbf{11.08}} & 34.25 \\
 & Reweighted & 62.41 & \textbf{67.79} & \textbf{87.47} & \textbf{19.67} & \textbf{97.27} & 96.68 & \textbf{84.97} & \textbf{11.71} & \textbf{31.38} \\
 & MinMax & \underline{\textbf{64.24}} & \underline{\textbf{67.98}} & 87.11 & \underline{\textbf{19.13}} & \textbf{97.24} & \textbf{97.36} & \textbf{85.70} & \textbf{11.67} & \underline{\textbf{30.79}} \\
\midrule
\multirow{4}{*}{Gemma-3 4B (it)} & Pretrained & \underline{60.31} & 62.59 & 84.89 & 22.30 & 97.18 & \underline{96.78} & \underline{90.28} & 6.50 & 28.81 \\
 & Conditional & 55.58 & 61.68 & \underline{\textbf{85.85}} & 24.16 & \underline{\textbf{98.21}} & 96.72 & 88.31 & 8.41 & 32.57 \\
 & Reweighted & 59.83 & \underline{\textbf{64.55}} & \textbf{85.47} & \underline{\textbf{20.91}} & \textbf{97.52} & 96.07 & 89.87 & \underline{\textbf{6.20}} & \underline{\textbf{27.11}} \\
 & MinMax & 59.10 & \textbf{63.97} & \textbf{85.62} & \textbf{21.65} & \textbf{97.85} & 96.08 & 89.79 & \textbf{6.28} & \textbf{27.93} \\
\bottomrule
\end{tabular}

%% file: tables/examples_by_model_tabularx.tex
\begingroup
\setlength{\tabcolsep}{6pt}
\renewcommand{\arraystretch}{1.15}
\scriptsize
\begin{tabularx}{\textwidth}{p{0.05\textwidth} p{0.10\textwidth}  >{\ttfamily}X p{0.10\textwidth}}
 
\toprule
\textbf{Model} & \textbf{Method} & \textbf{Text} & \textbf{Gender} \\
\midrule
1b & Pretrained & Joseph Louis Casimir Czerniak (Polish: [jzef luiz ksnak]; 15 February 1898 14 March 1943) was a Polish politician who served as the 20th Prime Minister of Poland from 1925 to 1928 and again from 1935 to 1939. He was the leader of the Polish Socialist Party, a Polish political party founded in 1919, and served as its leader from 1924 to 1935. He was the first Polish Prime Minister to be elected to the Polish Sejm. He was assassinated in 1943. Czerniak was born in Warsaw. He studied at the Warsaw University, graduating in 1921. He was a member of the Polish Socialist Party and participated in the PolishSoviet War. In 1925, he was appointed Prime Minister after the resignation of Stanisaw Wojciech Mikoaj Gombrzyski. He was re-elected in 1928, but was dismissed by the Sejm in 1935. He was assassinated in 1943. & male \\
1b & Pretrained & Robert Charles Stevenson (29 June 1847 14 September 1927) was an English actor, comedian and comedian who appeared in numerous pantomimes and musical comedies, and was best known for his starring roles in the 1912 film "The Tragedy of Esther Waters" and the 1915 film "The Homecoming of Robin Hood". He was born in London and began his career as a professional actor in the West End, and later moved to the United States, where he appeared in a number of stage shows and films. He died in 1927. & male \\
1b & Min-Max & Joseph Louis Broussard (November 19, 1881 July 8, 1962) was a Major League Baseball pitcher who played for the St. Louis Cardinals, Philadelphia Athletics, New York Yankees, and Chicago Cubs. He was a member of the Cardinals' 1922 World Series championship team and was a member of the 1926 National League pennant winning team. He was the first pitcher to throw a no-hitter in the National League, and was the first pitcher to throw a complete game shutout in the National League, accomplishing the feat in 1928 while pitching for the Yankees. Born in El Paso, Texas, Broussard began his baseball career with the Tucson Toros of the Pacific Coast League. He was signed by the Cardinals in 1909 and made his major league debut in 1910. He was traded to the Athletics in 1912, where he was released after the 1914 season. He returned to the Cardinals in 1915 and pitched for the team until his retirement in 1925. Broussard was inducted into the National Baseball Hall of Fame in 1962. & male \\
1b & Min-Max & Robert "Bob" C. Henningsen (born June 7, 1935) is an American retired professional baseball player who played in Major League Baseball for 16 seasons as a catcher. He played his entire career with the St. Louis Cardinals, and is a member of the Cardinals' franchise's Hall of Fame. Henningsen was a key player on the Cardinals' 1962 World Series championship team and was a three-time All-Star. He was a member of the Cardinals' 1963 National League pennant team. He also won a World Series championship in 1964 with the Cardinals and was named the World Series MVP. & male \\
1b & RW & Joseph Louis Casimir Czerniak (Polish: [jzef luiz ksnak]; 15 February 1898 14 March 1943) was a Polish politician who served as the 20th Prime Minister of Poland from 1925 to 1928 and again from 1935 to 1939. He was the leader of the Polish Socialist Party, a Polish political party founded in 1919, and served as its leader from 1924 to 1935. After the PolishSoviet War, he became a key figure in the Polish government and was appointed Prime Minister in 1935. He was the first Polish government leader to be imprisoned by the Nazis. He died in 1943 while in custody. Czerniak was born in the city of Krakw. He studied law at the University of Krakw, where he became involved in the Polish Socialist Party. He was a member of the Polish government in exile from 1925 to 1928, and was a key figure in the Polish government of 1935. He was imprisoned by the Nazis after the invasion of Poland in 1939 and died in custody in 1943. & male \\
1b & RW & Robert Charles Stevenson (29 November 1847 14 September 1927) was an English actor, comedian and comedian who appeared in numerous pantomimes and musical comedies, and was best known for his starring roles in the 1912 film "The Tragedy of Esther Waters" and the 1915 film "The Homecoming of Robin Hood". He was born in London and educated at Harrow School and the Royal Academy of Dramatic Art. He began his acting career in London's West End, performing in a number of pantomimes and musical comedies, before moving to the United States, where he became a popular actor and comedian in the early 20th century. He appeared in a number of films, and toured the United States in the 1910s and 1920s, performing in vaudeville and revues. He died in 1927. & male \\
3b & Pretrained & Joseph "Joe" John "Bull" Connolly (March 2, 1895 November 7, 1969) was an American professional baseball player who played in Major League Baseball (MLB) for 17 seasons, all with the New York Yankees. He was a left-handed pitcher and a member of the Yankees' 1923 World Series championship team. Connolly was the Yankees' ace pitcher in the 1923 season, posting a 207 winloss record with a 2.61 earned run average (ERA). He was also a key member of the Yankees' 1927 team that won the World Series, and pitched in 18 games for the team. Connolly was a member of the Yankees' rotation in the early 1920s, but his career was marked by inconsistency and a series of injuries. He was a member of the Yankees' starting rotation for only 18 of his 17 seasons in the major leagues, and he was out of baseball for several years due to various injuries. He was released by the Yankees in 1930, but returned to the team in 1934. He played for the Yankees until 1938, when he retired. & male \\
3b & Pretrained & Charles Charles Victor Sargent (March 21, 1841 May 25, 1927) was an American landscape gardener, horticulturist, and botanist. He was a professor of horticulture at the Massachusetts Agricultural College (now the University of Massachusetts Amherst) for over 40 years and was the first president of the American Horticultural Society. He is best known for his work on the development of the Boston Public Garden and the Boston botanical garden at Harvard University. Sargent was born in Boston, Massachusetts, and grew up in a family of modest means. He was apprenticed to a nurseryman at the age of 15, and after completing his apprenticeship, he worked in several other gardens in the United States. In 1872, he was appointed as a professor of horticulture at the Massachusetts Agricultural College, and he remained there until his retirement in 1913. Sargent was a prolific writer, and his books on horticulture were widely read in the United States and abroad. He was also a prolific writer on botanical topics, and his works were published in numerous journals and magazines. & male \\
3b & Min-Max & Joseph "Jack" Gorman (born 10 November 1985) is a former English professional footballer who played as a striker for several clubs, including Wigan Athletic, Ipswich Town and Crawley Town. Gorman began his career at Wigan Athletic, signing for the club in 2005 from Manchester City. He made his debut for the club in 2005 and scored one goal in 14 appearances. He was released from the club at the end of the 200506 season and joined Ipswich Town in 2006. He spent two years with the club, scoring six goals in 51 appearances before moving to Crawley Town in 2008. He spent two seasons with the club, scoring 21 goals in 69 appearances before leaving the club in 2010. & male \\
3b & Min-Max & Charles William Williams (6 July 1831 7 January 1910) was an English-born Australian cricketer who played first-class cricket for the New South Wales cricket team. He was a right-handed batsman and occasional wicket-keeper. Born in England, Williams moved to Australia with his family when he was eight years old. He was educated at the Sydney Church of England Grammar School and began his cricketing career with the school team, before making his first-class debut for New South Wales in 1853. He played for the team for the next 12 years, and was a member of the 185960 team that toured England, where he also played for the Marylebone Cricket Club. He was known for his batting, particularly his ability to score big centuries. He played in the first ever Test match in 1877 and was also a member of the New South Wales team that toured England in 188283. He retired from first-class cricket in 1865, but continued to play in minor league cricket until his death in 1910. & male \\
3b & RW & Joseph "Joe" O'Connell (May 19, 1920 August 24, 1997) was an American politician who served as a member of the New Jersey General Assembly from the 2nd district from 1958 to 1964 and from 1967 to 1971. He was a Democrat. O'Connell was born in New York City and graduated from the New York City College of Pharmacy. He served in the United States Army during World War II and was awarded the Purple Heart. He graduated from Rutgers University and later earned a master's degree in education from Columbia University. He was a teacher and a pharmacist before entering politics. He was elected to the New Jersey General Assembly in 1957 and served until 1964. He was defeated for re-election in 1964, and did not seek election again until 1966. He served another term until 1971, when he was defeated in a Republican primary. & male \\
3b & RW & Charles Francis "Charlie" Jones (February 27, 1886 January 15, 1968) was an American baseball player and manager in Major League Baseball. He played in Major League Baseball from 1909 to 1925, playing for the Philadelphia Athletics, New York Giants, and Boston Braves. He also managed the St. Louis Browns from 1928 to 1929. Jones was a four-time All-Star and won the National League batting title in 1912. He is a member of the National Baseball Hall of Fame. He is the only player in Major League Baseball history to have won a batting title, a Gold Glove Award, and a World Series championship. & male \\
\bottomrule
\end{tabularx}
\endgroup

%% file: tables/examples_by_model_tabularx_cond.tex
\begingroup
\setlength{\tabcolsep}{6pt}
\renewcommand{\arraystretch}{1.15}
\scriptsize
\begin{tabularx}{\textwidth}{p{0.05\textwidth} p{0.10\textwidth}  >{\ttfamily}X p{0.10\textwidth}}
 
\toprule
\textbf{Model} & \textbf{Method} & \textbf{Text} & \textbf{Gender} \\
\midrule
1b & Pretrained & Vivian Joseph O'Neal (born April 21, 1987) is an American professional baseball pitcher who played in Major League Baseball (MLB) as a relief pitcher for the New York Yankees. A former high school standout, he played college baseball for the University of Michigan Wolverines. He was drafted by the Yankees with the 8th round pick of the 2006 MLB draft. O'Neal was named the Yankees' Minor League Pitching Prospect of the Year in 2008. He was named a Golden Glove winner for his defensive play at shortstop in 2011. He was traded to the Los Angeles Angels in 2012. O'Neal was released by the Angels in 2013 and signed with the Baltimore Orioles in 2014. He was traded to the Kansas City Royals in 2015. He was traded to the Texas Rangers in 2016. He was released by the Rangers in 2017 and signed with the New York Yankees in 2018. He was traded to the Miami Marlins in 2019. He was released by the Marlins in 2020 and signed with the Seattle Mariners in 2021. & male \\
1b & Pretrained & Lucius Julius Caesar (c. 52 BC 44 BC) was a Roman general and statesman who served as dictator of Rome in 49 BC. He was a member of the Julian clan and a member of the Julian clan, a patrician family that played a significant role in Roman politics. Caesar was the son of Gaius Julius Caesar and Cornelia, a member of the Julian clan. He was appointed as a quaestor in 43 BC, and served in the Roman Senate in 42 BC, but was expelled from the Senate in 41 BC after he was accused of treason and exiled to the island of Corsica. He returned to Rome in 40 BC, and served as a quaestor in 39 BC. He was appointed dictator of Rome in 49 BC, and was elected dictator for the fourth time in 44 BC, becoming the first Roman dictator since 49 BC. He was assassinated in 44 BC. Caesar's reign was marked by a series of military victories against the Gauls, and he was celebrated for his military campaigns, which included the conquest of the Roman province of Hispania, the conquest of the island of Sicily, and the expulsion of the Parthian Empire from the Roman province of Cappadocia. He also reformed the Roman calendar, introducing the Julian calendar, and reformed the Roman Senate, establishing the office of the praetor maximus. & male \\
1b & Min-Max & Vivian James (ne Wilson; 4 October 1919 28 December 1994) was an English professional golfer who competed in the 1947 British Open and the 1951 US Open. She was the first woman to win a major championship. She was also the first woman to be ranked number one in the world golf rankings for a period of two years. She was a member of the first class of golfers to be inducted into the World Golf Hall of Fame. & female \\
1b & Min-Max & Lucius Marcellus Valerius Rufus (fl. 1st century BC) was a Roman senator and statesman from the gens Rufi, a patrician family that included several prominent politicians and military leaders in the Roman Republic. He was a member of the first century BC, and was a member of the Curia of the Roman Senate from 41 BC until his death in 1 BC. Rufus was known for his political and military career in the late Republic, during which he served as a general under the dictator Julius Caesar. He was appointed consul in 41 BC and served as governor of the province of Asia in 39 BC. He died in 1 BC. & male \\
1b & RW & Vivian Vess, ne (born 1956) is a Canadian actress, best known for her roles in the TV series The Shield and its spin-off series Alphas. She was born in New York City and began her career in television and film in the early 1980s. She began as a child actress on the soap opera The Young and the Restless and later appeared on the television series Perfect Strangers, The Love Boat, and The Golden Girls. She gained recognition for her roles in the television series The Shield and its spin-off series Alphas, for which she won a Canadian Screen Award for Best Supporting Actress in a Drama Series. & female \\
1b & RW & Lucius Quinctius Cincinnatus (c. 180c. 90 BCE) was a Roman statesman and military leader of the late Roman Republic. He was a member of the noble family of the Cincii, and was a prominent figure in the late Roman Republic, serving as a censor, quaestor, and tribune of the plebs. Cincinnatus was born in 180 BCE. He was a veteran of the Second Punic War, and served as a quaestor in the late 2nd century BCE. He was elected as a censors in 247 BCE, and then as a tribune of the plebs in 229 BCE. He served as censor again in 219 BCE, and was elected as a consul in 216 BCE. He was appointed as dictator in 214 BCE, and then as dictator again in 211 BCE. He was assassinated in 211 BCE, and succeeded by his nephew, Lucius Quinctius Calvus. & male \\
3b & Pretrained & Joseph "Joe" O'Connell (May 19, 1920 August 4, 1997) was an American politician who served as a member of the New Jersey General Assembly from the 9th district from 1958 to 1964 and from 1967 to 1971. He was a Democrat. O'Connell was born in New York City and graduated from the New York City High School of Commerce. He served in the United States Army during World War II and the Korean War, and worked as a laborer and a truck driver. He was elected to the General Assembly in 1957 and served until 1964. He was again elected in 1966 and served until 1971. & male \\
3b & Pretrained & Robert Charles "Bob" Jones III (born March 30, 1940) is an American businessman, politician, and Christian minister. He is best known for founding Liberty University, a Christian university in Lynchburg, Virginia, and for his advocacy of Christian nationalism and conservative politics. He was the chairman of the Board of Directors of the Family Research Council, a conservative Christian advocacy group, from 1982 until 2012. He was also the founder and chairman of the World Relief and Development Council, a Christian charity. Jones founded Liberty University in 1971 as a small Bible college in Lynchburg, Virginia, and expanded it into a large university by the 1980s. In the 1980s, he was a prominent figure in the Christian Right, and he ran for the Republican presidential nomination in 1988. He has been a vocal critic of the separation of church and state and has advocated for the teaching of creationism in public schools. He has been described as a "Christian fundamentalist" and a "evangelical". & male \\
3b & Min-Max & Joseph "Joe" Jones (February 28, 1897 December 23, 1971) was an American professional baseball player who played for the Philadelphia Athletics, Chicago Cubs, and Brooklyn Dodgers of Major League Baseball (MLB). He played as a pitcher and outfielder. Jones was born in New York City and attended George Washington University, where he played college baseball for the George Washington Pioneers. After college, he was signed by the Athletics and made his MLB debut in 1918. He played for the Athletics for eight seasons, including three All-Star appearances, and was a member of the 1927 World Series team. He was traded to the Cubs in 1929, where he played for two seasons, and was then traded to the Dodgers in 1931. He played for the Dodgers for four seasons before his MLB career was cut short by a series of injuries. & male \\
3b & Min-Max & Robert James "Bobby" Jones (23 July 1890 2 September 1971) was an English professional footballer who played as a centre-half. He made over 200 appearances for the first team of Southampton Football Club, and also played for the England national team. Born in Southampton, Jones began his career with local side East End United before joining Southampton in 1908. He made his first-team debut in 1910 and became a regular player for the club, helping them win the Southern League title in 191213. He also played for the England national team, earning 14 caps, and was part of the team that won the 1920 Summer Olympics. Jones moved to Fulham in 1914, but returned to Southampton in 1919, playing until his retirement in 1925. He later managed the club's youth team and was appointed chairman of the club in 1954. He was knighted in 1958 for his services to football. & male \\
3b & RW & Joseph "Joe" Lawler (born 2 January 1968) is an English former professional footballer who played as a midfielder for several clubs including Manchester City, Ipswich Town and Sunderland. He also had a brief spell at Middlesbrough. Lawler began his career with his hometown club Manchester City, before moving to Ipswich Town in 1990. He won the First Division title in his first season at the club, and played in the 1992 FA Cup final. He moved to Sunderland in 1993, where he won the First Division title again, and played in the 1994 FA Cup final. He moved to Middlesbrough in 1995, but left after just one season. He had a brief spell at Norwich City in 1997 before retiring. & male \\
3b & RW & Robert Francis "Bob" Johnson (January 23, 1894 August 7, 1971) was an American professional baseball player. He played in Major League Baseball (MLB) as a pitcher for the St. Louis Browns from 1913 to 1918. He was a left-handed thrower and batted and threw right-handed. Johnson was a member of the 1915 World Series championship team. He is best known for throwing a no-hitter in Game 1 of the 1915 World Series, and for his 1916 season in which he won 13 games and lost just two, and was named the American League leader in shutouts with 11. Johnson was also the American League leader in shutouts with 11 in 1917, and he was a member of the St. Louis Browns' 1918 World Series championship team. Johnson's career was cut short by an injury, as he suffered a shoulder injury in 1918 and was forced to retire from baseball. He later worked as a baseball scout for the Chicago Cubs and was also involved in the promotion of the St. Louis Browns. & male \\
\bottomrule
\end{tabularx}
\endgroup